\documentclass[dvipsnames]{article}
\usepackage{graphicx}
\usepackage{booktabs}
\usepackage{xcolor}
\usepackage{caption}
\usepackage{subcaption}
\usepackage{showlabels}
\usepackage{enumitem}

\usepackage{amsmath,amsfonts,bm}

\def\Eqref#1{Equation~\ref{#1}}
\def\eqref#1{\Eqref{#1}}

\def\1{\bm{1}}

\DeclareMathAlphabet{\mathsfit}{\encodingdefault}{\sfdefault}{m}{sl}
\SetMathAlphabet{\mathsfit}{bold}{\encodingdefault}{\sfdefault}{bx}{n}

\newcommand{\R}{\mathbb{R}}

\usepackage{subcaption}
\usepackage{bbm}
\usepackage{mathtools}
\usepackage{amsthm, amssymb}
\newtheorem{theorem}{Theorem}[section]
\newtheorem{lemma}{Lemma}%

\newtheorem{corollary}{Corollary}%
\newtheorem{remark}{Remark}%

\def\1{\bm{1}}
\newcommand{\D}{\mathcal{D}}
\newcommand{\X}{\mathcal{X}}
\newcommand{\Y}{\mathcal{Y}}
\newcommand{\T}{\mathcal{T}}
\newcommand{\K}{\mathcal{K}}
\newcommand{\bbE}{\mathbb{E}}
\newcommand{\flin}{f^{\textrm{lin}}}
\newcommand{\glin}{g^{\textrm{lin}}}
\newcommand{\dflin}{{{\dot f}_t^{\textrm {lin}}}}

\newcommand{\Id}{\bf Id} 
\newcommand{\sw}{\sigma_\omega}
\newcommand{\sws}{\sigma_\omega^2}

\newcommand{\sbs}{\sigma_b^2}

\newcommand{\pp}[1]{\left( #1 \right)}

\newcommand{\mc}{\mathcal}

\newcommand{\finntk}{\hat\Theta}
\newcommand{\infntk}{\Theta}
\newcommand{\infnngp}{\mathcal K}
\newcommand{\finnngp}{\hat {\mathcal K}}
\newcommand{\tpoint}{x}

\newcommand{\op}{\rm{op}}
\newcommand{\mins}{\lambda_{\rm{min}}}
\newcommand{\maxs}{\lambda_{\rm{max}}}

\newcommand{\lss}{R_0}

\newcommand{\relu}{{$\operatorname{ReLU}$}}
\newcommand{\erf}{$\operatorname{erf}$ }
\newcommand{\critical}{\eta_{{\rm critical}} }

\newcommand*\samethanks[1][\value{footnote}]{\footnotemark[#1]}

\definecolor{darkgreen}{rgb}{0,0.6,0}
\newcommand{\js}[1]{{\color{darkgreen}[JSD: #1]}}
\newcommand{\jp}[1]{{\color{red}[JP: #1]}}
\newcommand{\jl}[1]{{\color{cyan}[JL: #1]}}
\newcommand{\xl}[1]{{\color{blue}[XL: #1]}}
\newcommand{\yb}[1]{{\color{green}[YB: #1]}}
\newcommand{\scom}[1]{{\color{magenta}[SS: #1]}}
\newcommand{\rcom}[1]{{\color{orange}[RN: #1]}}

\renewcommand{\js}[1]{}
\renewcommand{\jp}[1]{}
\renewcommand{\jl}[1]{}
\renewcommand{\xl}[1]{}
\renewcommand{\yb}[1]{}
\renewcommand{\scom}[1]{}
\renewcommand{\rcom}[1]{}
\PassOptionsToPackage{numbers}{natbib}

\usepackage[final]{neurips_2019}

\usepackage{hyperref}
\newcommand{\papertitle}{Wide Neural Networks of Any Depth Evolve as Linear Models Under Gradient Descent}
\hypersetup{
    unicode=false,
    pdftoolbar=true,
    pdfmenubar=true,
    pdffitwindow=false,
    pdfstartview={FitH},
    pdftitle={\papertitle},
    pdfauthor={Jaehoon Lee*, Lechao Xiao*, Samuel S. Schoenholz, Yasaman Bahri, Roman Novak, Jascha Sohl-Dickstein, Jeffrey Pennington},
    pdfsubject={\papertitle},
    pdfcreator={Jaehoon Lee∗, Lechao Xiao∗, Samuel S. Schoenholz, Yasaman Bahri, Roman Novak, Jascha Sohl-Dickstein, Jeffrey Pennington},
    pdfproducer={Jaehoon Lee∗, Lechao Xiao∗, Samuel S. Schoenholz, Yasaman Bahri, Roman Novak, Jascha Sohl-Dickstein, Jeffrey Pennington},
    pdfkeywords={Gradient Descent, Neural Network, Training Dynamics},
    pdfnewwindow=true,
    colorlinks=True,
    linkcolor=BurntOrange,
    citecolor=RawSienna,
    filecolor=magenta,
    urlcolor=cyan
}

\newcommand{\sref}[1]{\S\ref{#1}}

\usepackage[utf8]{inputenc} %
\usepackage[T1]{fontenc}    %
\usepackage{hyperref}       %
\usepackage{url}            %
\usepackage{booktabs}       %
\usepackage{amsfonts}       %
\usepackage{nicefrac}       %
\usepackage{microtype}      %

\title{\papertitle}

\newcommand{\email}[1]{\tt\small\href{mailto:#1@google.com}{#1}}

\author{
Jaehoon Lee\thanks{Both authors contributed equally to this work. Work done as a member of the Google AI Residency program (\href{https://g.co/airesidency}{https://g.co/airesidency}).}\, ,\,\, Lechao Xiao\samethanks[1]\,,\,\,\,\, Samuel S. Schoenholz,\,\,\,\, Yasaman Bahri\\[0.1cm] \textbf{Roman Novak,\,\,\,\, Jascha Sohl-Dickstein,\,\,\,\, Jeffrey Pennington}\\[0.2cm]
Google Brain\\[0.2cm]
  \texttt{\{\email{jaehlee}, \email{xlc}, \email{schsam}, \email{yasamanb}, \email{romann}, \email{jaschasd}, \email{jpennin}\}@google.com} \\
}

\begin{document}

\maketitle

\begin{abstract}
 A longstanding goal in deep learning research has been to precisely characterize training and generalization. However, the often complex loss landscapes of neural networks have made a theory of learning dynamics elusive. In this work, we show that for wide neural networks the learning dynamics simplify considerably and that, in the infinite width limit, they are governed by a linear model obtained from the first-order Taylor expansion of the network around its initial parameters.
Furthermore, mirroring the correspondence between wide Bayesian neural networks and Gaussian processes, gradient-based training of wide neural networks with a squared loss produces test set predictions drawn from a Gaussian process with a particular compositional kernel. While these theoretical results are only exact in the infinite width limit, we nevertheless find excellent empirical agreement between the predictions of the original network and those of the linearized version even for finite practically-sized networks. This agreement is robust across different architectures, optimization methods, and loss functions.
\end{abstract}

\section{Introduction}

Machine learning models based on deep neural networks have achieved unprecedented performance across a wide range of tasks \cite{Krizhevsky2012, he2016deep, devlin2018bert}.
Typically, these models are regarded as complex systems for which many types of theoretical analyses are intractable.
Moreover, characterizing the gradient-based training dynamics of these models is challenging owing to the typically high-dimensional non-convex loss surfaces governing the optimization. As is common in the physical sciences, investigating the extreme limits of such systems can often shed light on these hard problems. For neural networks, one such limit is that of infinite width, which refers either to the number of hidden units in a fully-connected layer or to the number of channels in a convolutional layer.
Under this limit, the output of the network at initialization is a draw from a Gaussian process~(GP); moreover, the network output remains governed by a GP after exact Bayesian training using squared loss~\citep{neal,lee2018deep,matthews2018,novak2018bayesian, garriga2018deep}. Aside from its theoretical simplicity, the infinite-width limit is also of practical interest as wider networks have been found to generalize better~\cite{lee2018deep, novak2018bayesian,neyshabur2014search, novak2018sensitivity, neyshabur2018the}. 

In this work, we explore the learning dynamics of wide neural networks under gradient descent and find that the weight-space description of the dynamics becomes surprisingly simple: as the width becomes large, the neural network can be effectively replaced by its first-order Taylor expansion with respect to its parameters at initialization. For this 
linear model, the dynamics of gradient descent become \emph{analytically tractable}. While the linearization is only exact in the infinite width limit, we nevertheless find excellent agreement between the predictions of the original network and those of the linearized version even for finite width configurations. The agreement persists across different architectures, optimization methods, and loss functions. 

For squared loss, the exact learning dynamics admit a closed-form solution that allows us to characterize the evolution of the predictive distribution in terms of a GP. This result can be thought of as an extension of ``sample-then-optimize" posterior sampling~\cite{matthews2017sample} to the training of deep neural networks. Our empirical simulations confirm that the result accurately models the variation in predictions across an ensemble of finite-width models with different random initializations.

Here we summarize our contributions:
\begin{itemize}[leftmargin=*]
    \item \textbf{Parameter space dynamics}: 
    We show that wide network training dynamics in parameter space are equivalent to the training dynamics of a model which is affine in the collection of all network parameters, the weights and biases. This result holds regardless of the choice of loss function. For squared loss, the dynamics admit a closed-form solution as a function of time.
    
    \item \textbf{Sufficient conditions for linearization}: We formally prove that there exists a threshold learning rate $\critical$ (see Theorem \ref{thm:main}), such that 
    gradient descent training trajectories 
    with learning rate smaller than $\critical$ stay in an $\mathcal O\left(n^{-1/ 2}\right)$-neighborhood of the trajectory of the linearized network when $n$, the width of the hidden layers, is sufficiently large.  

    \item \textbf{Output distribution dynamics}: 
    We formally show that the predictions of a neural network throughout gradient descent training are described by a GP as the width goes to infinity (see Theorem \ref{thm:distribution}), extending results from \citet{Jacot2018ntk}. 
    We further derive explicit time-dependent expressions for the evolution of this GP during training. 
    Finally, we provide a novel interpretation of the result. In particular, it offers a quantitative understanding of the mechanism by which gradient descent differs from Bayesian posterior sampling of the parameters: while both methods generate draws from a GP, gradient descent does not generate samples from the posterior of any probabilistic model. 
    
    \item \textbf{Large scale experimental support}: We empirically investigate the applicability of the theory in the finite-width setting and find that it gives an accurate characterization of both learning dynamics and posterior function distributions across a variety of conditions, including some practical network architectures such as the wide residual network~\cite{zagoruyko2016wide}.

    \item \textbf{Parameterization independence}: We note that linearization result holds both in standard and NTK parameterization (defined in \sref{sec:notation}), 
    while previous work assumed the latter, 
    emphasizing that the effect is due to increase in width rather than the particular parameterization. 
    
    \item \textbf{Analytic \relu{} and \erf{} neural tangent kernels}: We compute the analytic neural tangent kernel corresponding to fully-connected networks with \relu{} or \erf{} nonlinearities.

    \item \textbf{Source code}: 
    Example code investigating both function space and parameter space linearized learning dynamics described in this work is released as open source code within~\cite{neuraltangents2019}.\footnote{Note that the open source library has been expanded since initial submission of this work.}
    We also provide accompanying interactive Colab notebooks for both  \href{https://colab.sandbox.google.com/github/google/neural-tangents/blob/master/notebooks/weight_space_linearization.ipynb}{\bf parameter space}\footnote{\scriptsize \href{https://colab.sandbox.google.com/github/google/neural-tangents/blob/master/notebooks/weight_space_linearization.ipynb}{colab.sandbox.google.com/github/google/neural-tangents/blob/master/notebooks/weight\_space\_linearization.ipynb}} and \href{https://colab.sandbox.google.com/github/google/neural-tangents/blob/master/notebooks/function_space_linearization.ipynb}{\bf function space}\footnote{\scriptsize \href{https://colab.sandbox.google.com/github/google/neural-tangents/blob/master/notebooks/function_space_linearization.ipynb}{colab.sandbox.google.com/github/google/neural-tangents/blob/master/notebooks/function\_space\_linearization.ipynb}} linearization.
    
\end{itemize}

\subsection{Related work}

We build on recent work by~\citet{Jacot2018ntk} that characterize the exact dynamics of network outputs throughout gradient descent training in the infinite width limit. Their results establish that full batch gradient descent in parameter space corresponds to kernel gradient descent in function space with respect to a new kernel, the Neural Tangent Kernel (NTK). 
We examine
what this implies about 
dynamics in parameter space, where training updates are actually made.

\citet{daniely2016} study the relationship between neural networks and kernels at initialization. They bound the difference between the infinite width kernel and the empirical kernel at finite width $n$, which diminishes as $\mathcal{O}(1/\sqrt{n})$. \citet{daniely2017sgd} uses the same kernel perspective to study stochastic gradient descent (SGD) training of neural networks.

\citet{saxe2013exact} study the training dynamics of deep linear networks, in which the nonlinearities are treated as identity functions. Deep linear networks are linear in their inputs, but not in their parameters. 
In contrast, we show that the outputs of sufficiently wide neural networks are linear in the updates to their parameters during gradient descent, but not usually their inputs.

\citet{du2018gradient, allen2018convergence-fc, allen2018convergence-rnn, zou2018stochastic} study the convergence of gradient descent to global minima.
They proved that for i.i.d. Gaussian initialization, the parameters of sufficiently wide networks move little from their initial values during SGD.
This small motion of the parameters is crucial to the effect we present, where wide neural networks behave linearly in terms of their parameters throughout training.

\citet{mei2018mean, chizat2018global,rotskoff2018neural,sirignano2018mean} analyze the mean field SGD dynamics of training neural networks in the large-width limit. Their mean field analysis describes distributional dynamics of network parameters via a PDE. However, their analysis is restricted to one hidden layer networks with a scaling limit $\left(1/n\right)$ different from ours $\left(1/\sqrt{n}\right)$, which is commonly used in modern networks~\cite{he2016deep, glorot2010understanding}. 

\citet{chizat2018note}\footnote{We note that this is a concurrent work and an expanded version of this note is presented in parallel at NeurIPS 2019.} argued that infinite width networks are in `lazy training' regime and maybe too simple to be applicable to realistic neural networks. Nonetheless, we empirically investigate the applicability of the theory in the finite-width setting and find that it gives an accurate characterization of both the learning dynamics and posterior function distributions across a variety of conditions, including some practical network architectures such as the wide residual network~\cite{zagoruyko2016wide}.

\section{Theoretical results}
\label{sec:TheoryResults}

\subsection{Notation and setup for architecture and training dynamics}
\label{sec:notation}
    Let $\D \subseteq \mathbb R^{n_0} \times \mathbb R^{k}$ denote the training set and $\X=\left\{x: (x,y)\in \D\right\}$ and $\Y=\left\{y: (x,y)\in \D\right\}$ denote the inputs and labels, respectively. Consider a fully-connected feed-forward network with $L$
    hidden
    layers with widths $n_{l}$, for $l = 1, ..., L$ and a readout layer with $n_{L+1} = k$.  For each $x\in\mathbb R^{n_0}$, we use $h^l(x), x^l(x)\in\mathbb R^{n_l}$ to represent the pre- and post-activation functions at layer $l$ with input $x$. The recurrence relation for a feed-forward network is defined as 
    \begin{align}
    \label{eq:recurrence}
    \begin{cases}
        h^{l+1}&=x^l W^{l+1} + b^{l+1}
        \\
        x^{l+1}&=\phi\left(h^{l+1}\right) 
        \end{cases}
        \,\, \textrm{and} 
        \,\,
        \begin{cases}
      W^{l}_{i, j}& = \frac {\sigma_\omega} {\sqrt{n_l}}  \omega_{ij}^l 
        \\
        b_j^l &= \sigma_b  \beta_j^l
    \end{cases}
    ,
    \end{align}
    where $\phi$ is a point-wise activation function, $W^{l+1}\in \mathbb R^{n_l\times n_{l+1}}$ and $b^{l+1}\in\mathbb R^{n_{l+1}}$ are the weights and biases, $\omega_{ij}^l$ and $ b_j^l $ are the trainable variables, drawn i.i.d. from a standard Gaussian $ \omega_{ij}^l,  \beta_{j}^l\sim \mathcal N(0, 1)$ at initialization, and $\sws$ and $\sbs$ are weight and bias variances. Note that this parametrization 
    is non-standard,
    and we will refer to it as the NTK parameterization. It has already been adopted in several recent works~\cite{van2017l2, karras2018progressive, Jacot2018ntk, du2018gradient, parkoptimal}. Unlike the standard parameterization that only normalizes the forward dynamics of the network, the NTK-parameterization also normalizes its backward dynamics. 
    We note that the predictions and training dynamics of NTK-parameterized networks are identical to those of standard networks, up to a width-dependent scaling factor in the learning rate for each parameter tensor. 
    As we derive, and support experimentally, in Supplementary Material (SM) \sref{sec:compare-parameterization} and \sref{sec: converge proof}, 
    our results (linearity in weights, GP predictions) also hold for 
    networks with a standard parameterization.
    
    We define 
    $\theta^l\equiv \operatorname{vec}\pp{\{ W^l, b^l\} }$,
    the $\pp{(n_{l-1}+1) n_l} \times 1$ vector of all parameters for layer $l$. 
    $\theta = \operatorname{vec}\pp{\cup_{l=1}^{L+1}{\theta^l}}$ then indicates the vector of all network parameters, with similar definitions for $\theta^{\leq l}$ and $\theta^{>l}$. Denote by $\theta_t$ the time-dependence of the parameters and by $\theta_0$ their initial values. We use $f_t(x) \equiv h^{L+1}(x)\in \mathbb R^{k}$ to denote the output (or logits) of the neural network at time $t$. 
    Let $\ell (\hat y, y):\mathbb R^{k}\times \mathbb R^{k}\to\mathbb{R}$ denote the loss function where the first argument is the prediction and the second argument the true label. In supervised learning, one is interested in learning a $\theta$ that minimizes the empirical loss\footnote{To simplify the notation for later equations, we use the \emph{total} loss here instead of the {\it average} loss, but for all plots in \sref{sec:experiments}, we show the \emph{average} loss.}, $\mathcal L = \sum_{(x, y)\in\D} \ell (f_t(x, \theta), y).$
    
    Let $\eta$ be the learning rate\footnote{Note that compared to the conventional parameterization, $\eta$ is larger by factor of width~\cite{parkoptimal}. The NTK parameterization allows usage of a universal learning rate scale irrespective of network width.}.
    Via continuous time gradient descent,
    the evolution of the parameters $\theta$ and the logits $f$ can be written as 
    \begin{align}
        \label{eq:nn-gradient-descent-weights}
        &\dot \theta_t = 
        - \eta  {\nabla_\theta f_t(\X)}^T
        \nabla_{f_t(\mathcal{X})} \mc L
        \\
        &\dot f_t(\X) = {\nabla_\theta f_t(\X)}\, \dot \theta_t 
        = - \eta  \, \finntk_t (\X, \X)  \nabla_{f_t(\mathcal{X})} \mc L
        \label{eq:nn-gradient-descent-outputs}
    \end{align}  
    where 
    $f_t(\X) = \operatorname{vec}\pp{\left[ f_t\pp{x} \right]_{x\in\X}}$, the $k|\D|\times 1$ vector of concatenated logits for all examples, and 
    $\nabla_{f_t(\mathcal{X})} \mc L$ is the gradient of the loss with respect to the model's output, $f_t(\mathcal{X})$.
    $\finntk_t \equiv \finntk_t(\X, \X) $ is the tangent kernel at time $t$, which is a $k|\D|\times k|\D|$ matrix
    \begin{align}\label{eq:tangent-kernel}
    \finntk_t &=
        {\nabla_\theta f_t(\X)} {\nabla_\theta f_t(\X)}^{T}= \sum_{l=1}^{L+1} {\nabla_{\theta^l} f_t(\X)} {\nabla_{\theta^l} f_t(\X)}^{T}. 
    \end{align}
    One can define the tangent kernel for general arguments, e.g. $\finntk_t(x, \X)$ where $x$ is test input. At finite-width, $\hat{\Theta}$ will depend on the specific random draw of the parameters and in this context we refer to it as the \emph{empirical} tangent kernel.
    
    The dynamics of discrete gradient descent
    can be obtained by replacing $\dot \theta_t$ and $\dot f_t(\X)$ with $(\theta_{i+1} - \theta_i)$ and $(f_{i+1}(\X) -f_i(\X))$ above, and replacing $e^{-\eta\finntk_0t}$ with $(1 - (1-\eta\finntk_0)^i)$ below. 
    
    \subsection{Linearized networks have closed form training dynamics for parameters and outputs}
    In this section, we consider the training dynamics of the linearized network. Specifically, we replace the outputs of the neural network by their first order Taylor expansion,
    \begin{align}
           \flin_{t}(x)\equiv f_{0}(x) + \left.{\nabla_\theta f_0(x)}\right\vert_{\theta=\theta_0}\, \omega_t\,, %
    \end{align}
    where $\omega_t \equiv \theta_t-\theta_0$ is the change in the parameters from their initial values. 
    Note that $\flin_{t}$ is the sum of two terms: the first term is the initial output of the network, which remains unchanged during training, and the second term captures the change to the initial value during training. The dynamics of gradient flow using this linearized function are governed by,
    \begin{align}
        &\dot \omega_t 
        = -  \eta  {\nabla_\theta f_0(\X)}^T
        \nabla_{\flin_t(\mathcal{X})} \mc L
        \label{eq:lin-nn-gradient-descent-weights}
        \\
        &\dflin(x) 
        = - \eta  \, \finntk_0 (x, \X)  \nabla_{\flin_t(\mathcal{X})} \mc L\,.
        \label{eq:lin-nn-gradient-descent-outputs}
    \end{align}
    As ${\nabla_\theta f_0(x)}$ remains constant throughout training, these dynamics are often quite simple.
    In the case of an MSE loss, i.e., $\ell(\hat y , y) = \frac 1 2 \|\hat y -y\|_2^2$, the ODEs have closed form solutions 
    \begin{align}
    &\omega_t  =  - {\nabla_\theta f_0(\X)}^T \finntk_0^{-1}\left(I - e^{- \eta \finntk_0 t}\right)\left(f_{0}(\X) - \Y\right)\,, \label{eq:lin-dynamics-weights}
    \\
    &\flin_{t}(\X)=(I - e^{- \eta\finntk_0 t})\Y + e^{-\eta \finntk_0 t}f_{0}(\X) \,. \label{eq:lin-dynamics-outputs}
    \end{align}
    For an arbitrary point $\tpoint$, $\flin_t(\tpoint) = \mu_t(\tpoint) +\gamma_t(\tpoint)$, where 
    \begin{align}
    \label{eq:flin-x}
    &\mu_t(\tpoint) = \finntk_0(\tpoint, \X)\finntk_0^{-1}\left(I- e^{-\eta\finntk_0 t}\right)\Y 
    \\
    \label{eq:flin-x-2}
    &\gamma_t(\tpoint) = f_{0}(\tpoint)-\finntk_0\left(\tpoint, \X\right)\finntk_0^{-1}\left(I\!-\!e^{- \eta \finntk_0 t}\right)f_{0}(\X).
    \end{align}
    Therefore, we can obtain the time evolution of the linearized neural network without 
    running gradient descent. 
    We only need to compute the tangent kernel $\finntk_0$ and the outputs $f_0$ at initialization and use Equations \ref{eq:lin-dynamics-weights}, \ref{eq:flin-x}, and \ref{eq:flin-x-2} to compute the dynamics of the weights and the outputs.
    
    \subsection{Infinite width limit yields a Gaussian process}
    As the width of the hidden layers approaches infinity, the Central Limit Theorem (CLT) implies that the outputs at initialization $\left\{f_{0}(x)\right\}_{x\in\X}$ converge to a multivariate Gaussian in distribution. Informally, this occurs because the pre-activations at each layer are a sum of Gaussian random variables (the weights and bias), and thus become a Gaussian random variable themselves. 
    See 
    \cite{poole2016exponential,schoenholz2016, lee2018deep, xiao18a, yang2017} for more details, and \cite{matthews2018b_arxiv, novak2018bayesian}
    for a formal treatment. 
    
    Therefore, randomly initialized neural networks are in correspondence with a certain class of GPs (hereinafter referred to as NNGPs), which facilitates a fully Bayesian treatment of neural networks \citep{lee2018deep,matthews2018}. More precisely, let $f_t^{i}$ denote the $i$-th output dimension and $\infnngp$ denote the  sample-to-sample kernel function (of the pre-activation) of the outputs in the infinite width setting, 
    \begin{align}
        \infnngp^{i, j}(x,x') = 
        \lim_{\min\pp{n_{1}, \dots, {n_L}}\to\infty}
        \mathbb E \left[ f_0^i(x)\cdot f_0^j(x')\right],
    \end{align}
    then $f_{0}(\X) \sim \mathcal{N}(0, \infnngp(\X, \X))$, where $\infnngp^{i, j}(x, x')$ 
    denotes the covariance between the $i$-th output of $x$ and $j$-th output of $x'$, 
    which can be computed recursively (see \citet[\S 2.3]{lee2018deep} and SM \sref{sec:KernelDerivation}).
    For a test input $\tpoint\in \X_T$, the joint output distribution $f\left([\tpoint, \X]\right)$ is also multivariate Gaussian.
    Conditioning on the training samples\footnote{
    This imposes that $h^{L+1}$ directly corresponds to the network predictions. 
    In the case of softmax readout,
    variational or sampling methods are required to marginalize over $h^{L+1}$.
    }, $f(\X)=\Y$, the 
    distribution of $\left.f(\tpoint)\right\vert \X, \Y$ is also a Gaussian $\mathcal N \left(\mu(\tpoint), \Sigma(\tpoint)\right)$, 
    \begin{align}
    \label{eq:nngp-exact-posterior}
    \mu(\tpoint) = \infnngp(\tpoint, \X) \infnngp^{-1}\Y, \quad
    \Sigma(\tpoint) = \infnngp(\tpoint, \tpoint) - \infnngp(\tpoint, \X) \infnngp^{-1}\infnngp(\tpoint, \X)^T
    ,
    \end{align}
    and where $\infnngp = \infnngp(\X, \X)$.  
    This is the posterior predictive distribution resulting from exact Bayesian inference in an infinitely wide neural network.
    
    \subsubsection{Gaussian processes from gradient descent training}
    
    If we freeze the variables $\theta^{\leq L}$ after initialization and only optimize $\theta^{L+1}$, the original network and its linearization are identical. Letting the width approach infinity, this particular tangent kernel $\finntk_0$ will converge to $\infnngp$ in probability 
    and \eqref{eq:flin-x} will converge to the posterior \eqref{eq:nngp-exact-posterior} as $t\to\infty$ (for further details see SM \sref{sec:gradient-readout-layer}). 
    This is a realization of the ``sample-then-optimize" approach for evaluating the posterior of a Gaussian process proposed in \citet{matthews2017sample}.
    
    If none of the variables are frozen, in the infinite width setting, $\finntk_0$ also converges in probability
    to a deterministic kernel $\infntk$ \cite{Jacot2018ntk, yang2019scaling}, which we sometimes refer to as the analytic kernel, and which can also be computed recursively (see SM \sref{sec:KernelDerivation}). For \relu{} and \erf{} nonlinearity, $\infntk$ can be exactly computed (SM \sref{sec:analytic_kernel}) which we use in \sref{sec:experiments}. Letting the width go to infinity, for any $t$, the output $\flin_t(\tpoint)$ of the linearized network is also Gaussian distributed because Equations \ref{eq:flin-x} and \ref{eq:flin-x-2} describe an affine transform of the Gaussian $[f_0(\tpoint), f_0(\X)]$. Therefore

    \begin{corollary}\label{cor:lin-distribution}
    For every test points in $x \in \X_T$, %
    and $t \geq 0$, $\flin_t(\tpoint)$ converges in distribution as width goes to infinity to a Gaussian with mean and covariance given by\footnote{Here {``\it+h.c.'' } is an abbreviation for ``plus the Hermitian conjugate''.}  
    \small
        \begin{align}
      &\mu(\X_T) =\infntk\left(\X_T, \X\right)\infntk^{-1}\left(I -e^{- \eta \infntk t}\right)\Y \,,
      \label{eq:lin-exact-dynamics-mean}
      \\
      &\Sigma(\X_T, \X_T) =\infnngp\left(\X_T, \X_T\right) +\infntk(\X_T, \X)\infntk^{-1}\left(I-e^{- \eta \infntk t}\right) \infnngp \left(I - e^{-\eta \infntk t}\right) \infntk^{-1} \infntk\left(\X, \X_T \right)\nonumber \\
      &\phantom{\Sigma(\X_T, \X_T) =\infnngp\left(\X_T, \X_T\right)} -\left(\infntk(\X_T, \X)\infntk^{-1}\left(I-e^{- \eta \infntk t}\right) \infnngp\left(\X, \X_T \right) + h.c. \right).
    \label{eq:lin-exact-dynamics-var}
    \end{align}
    \normalsize 
    Therefore, over random initialization, $\lim_{t\to\infty}\lim_{n\to\infty}\flin_t(x)$ has distribution   
    \begin{align}\label{eq:lin-exact-dynamics-var_inf}
    \mathcal N\big(&\infntk\left(\X_T, \X\right)\infntk^{-1}\Y,  \nonumber\\
    &\infnngp\left(\X_T, \X_T\right) +\infntk(\X_T, \X)\infntk^{-1}\infnngp \infntk^{-1} \infntk\left(\X, \X_T \right)- \left(\infntk(\X_T, \X)\infntk^{-1}\infnngp\left(\X, \X_T\right) + h.c. \right)\big).
    \end{align}
    
    \end{corollary}    
    
    Unlike the case when only $\theta^{L+1}$ is optimized, Equations~\ref{eq:lin-exact-dynamics-mean} and \ref{eq:lin-exact-dynamics-var} 
    do not admit an interpretation corresponding to the posterior sampling of a probabilistic model.\footnote{One possible exception is when the NNGP kernel and NTK are the same up to a scalar multiplication. This is the case when the activation function is the identity function and there is no bias term.} 
    We contrast the predictive distributions from the NNGP, NTK-GP (i.e. Equations \ref{eq:lin-exact-dynamics-mean} and \ref{eq:lin-exact-dynamics-var}) and ensembles of NNs in Figure~\ref{fig:posterior-dynamics}. 
    
    Infinitely-wide neural networks open up ways to study deep neural networks both under fully Bayesian training through the Gaussian process correspondence, and under GD training through the linearization perspective. The resulting distributions over functions are inconsistent (the distribution resulting from GD training does not generally correspond to a Bayesian posterior). We believe understanding the biases over learned functions induced by different training schemes and architectures is a fascinating avenue for future work.
   
    \subsection{Infinite width networks are linearized networks}
    \label{sec:Justification}
            \eqref{eq:nn-gradient-descent-weights} and \ref{eq:nn-gradient-descent-outputs} of the original network are intractable in general, since $\finntk_t$ evolves with time. However, for the mean squared loss, we are able to prove formally that, as long as the learning rate $\eta< \critical :=2({\mins(\infntk) + \maxs(\infntk)})^{-1}$, where ${\lambda_{\textrm{min/max}}}(\infntk)$ is  the min/max eigenvalue of $\infntk$, the gradient descent dynamics of the original neural network falls into its linearized dynamics regime. 
            
            \begin{theorem}[Informal]\label{thm:main}
            Let $n_1 =\dots =\ n_L=n$ and assume $\mins(\infntk)>0$. Applying gradient descent with learning rate $\eta < \critical$ (or gradient flow), for every $x\in \mathbb R^{n_0}$ with $\|x\|_2\leq 1$, with probability arbitrarily close to 1 over random initialization,  
            \begin{align}\label{eq:discrepancy-training}
            \sup_{t\geq 0}\left\|f_t(x) - \flin_t(x)\right\|_2, 
            \,\,\sup_{t\geq 0}\frac{\left\|\theta_t -\theta_0\right\|_2}{\sqrt n},
            \,\, 
            \sup_{t\geq 0}\left\|\finntk_t -  \finntk_0\right\|_F = \mathcal O(n^{-\frac 1 2}), \,\, {\rm as }\quad n\to \infty\,. 
            \end{align}
            \end{theorem} 
            Therefore, as $n\to\infty$, the distributions of $f_t(x)$ and $\flin_t(x)$ become the same.  Coupling with Corollary \ref{cor:lin-distribution}, we have 
            \begin{theorem}\label{thm:distribution}
            If $\eta < \eta_{\rm critical}$, then for every $x\in\R^{n_0}$ with $\|x\|_2\leq 1$, 
            as $n\to\infty$, $f_t(x)$ converges in distribution to the Gaussian with mean and variance given by \eqref{eq:lin-exact-dynamics-mean} and \eqref {eq:lin-exact-dynamics-var}.  
            \end{theorem}
            We refer the readers to Figure ~\ref{fig:posterior-dynamics} for empirical verification of this theorem.  
            The proof of Theorem \ref{thm:main} consists of two steps. The first step is to prove the global convergence of overparameterized neural networks \citep{du2018gradient, allen2018convergence-fc, allen2018convergence-rnn, zou2018stochastic} and stability of the NTK under gradient descent (and gradient flow);
            see SM \sref{sec: converge proof}. This stability was first observed and proved in \cite{Jacot2018ntk} in the gradient flow and sequential limit (i.e. letting $n_1\to\infty$, \dots, $n_L\to \infty$ sequentially) setting under certain assumptions about global convergence of gradient flow. 
            In \sref{sec: converge proof}, we show how to use the NTK to provide a self-contained (and cleaner) proof of such global convergence and the stability of NTK simultaneously. 
            The second step is to couple the stability of NTK with Gr\"{o}nwall's type arguments~\cite{dragomir2003some} to upper bound the 
            discrepancy between $f_t$ and $\flin_t$, i.e. the first norm in  \eqref{eq:discrepancy-training}. 
            Intuitively, the ODE of the original network (\eqref{eq:nn-gradient-descent-outputs}) can be considered as a $\|\finntk_t -  \finntk_0\|_F$-fluctuation from the linearized ODE (\eqref{eq:lin-nn-gradient-descent-outputs}). One expects the difference between the solutions of these two ODEs to be upper bounded by some functional of $\|\finntk_t -  \finntk_0\|_F$; see SM \sref{sec:sup-discrepancy}. 
            Therefore, for a large width network, the training dynamics can be well approximated by linearized dynamics. 

\begin{figure}[t]
  \centering
  
  \begin{subfigure}[b]{0.24\columnwidth}
  \includegraphics[width=\textwidth]{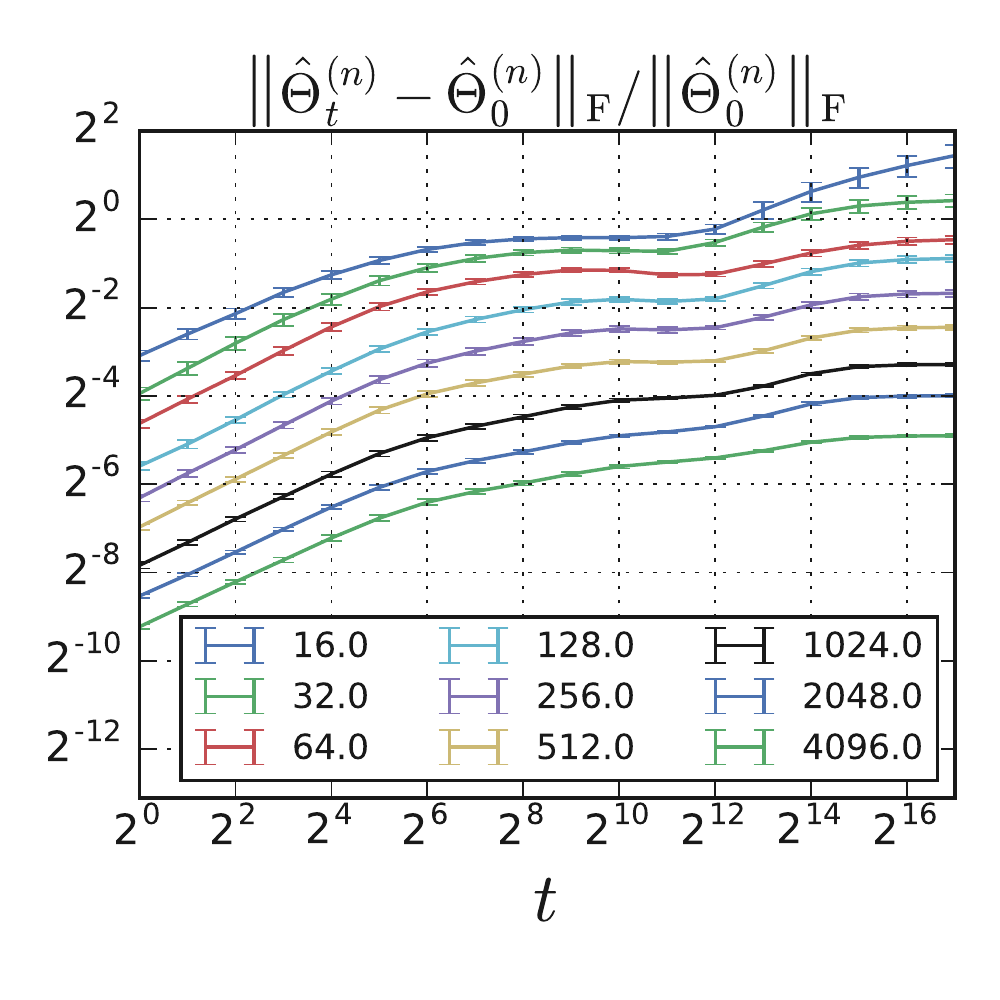} 
  \end{subfigure}
  \begin{subfigure}[b]{0.24\columnwidth}
  \includegraphics[width=\textwidth]{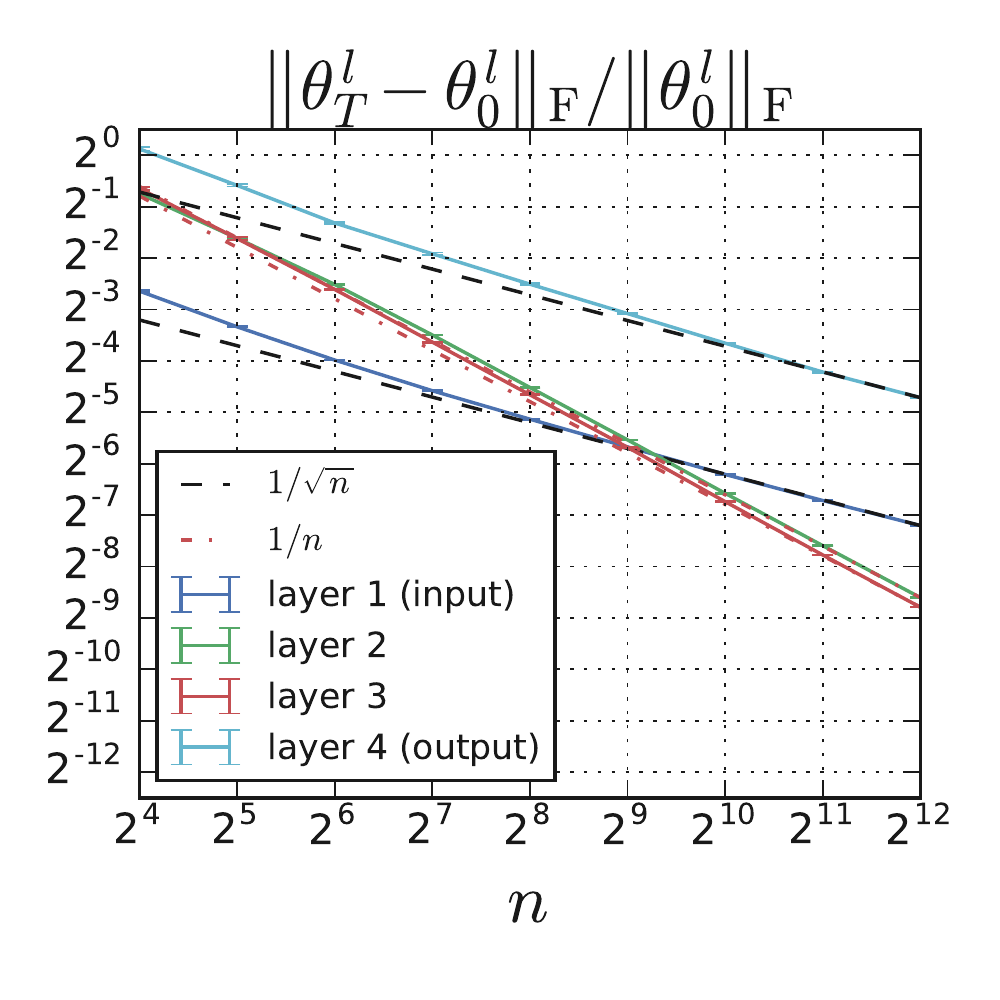} 
  \end{subfigure}
  \begin{subfigure}[b]{0.24\columnwidth}
  \includegraphics[width=\textwidth]{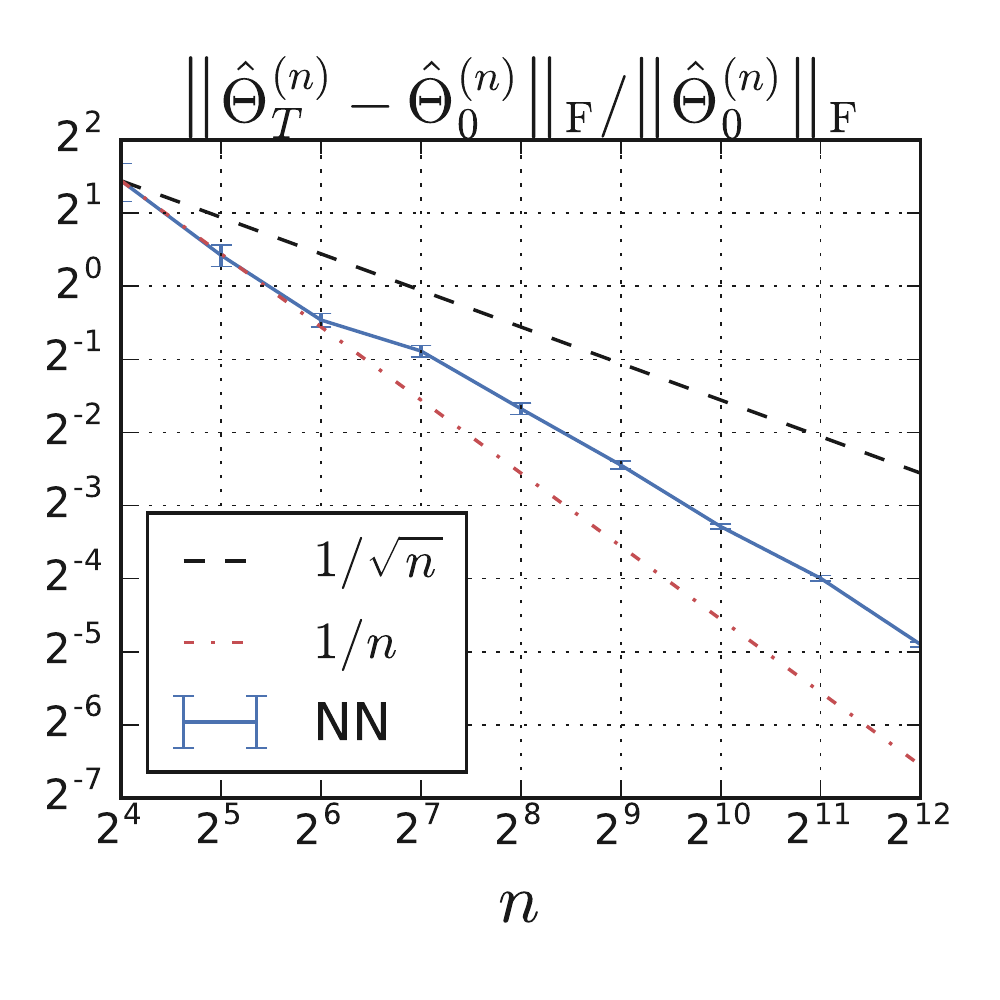} 
  \end{subfigure}
  \begin{subfigure}[b]{0.24\columnwidth}
  \includegraphics[width=\textwidth]{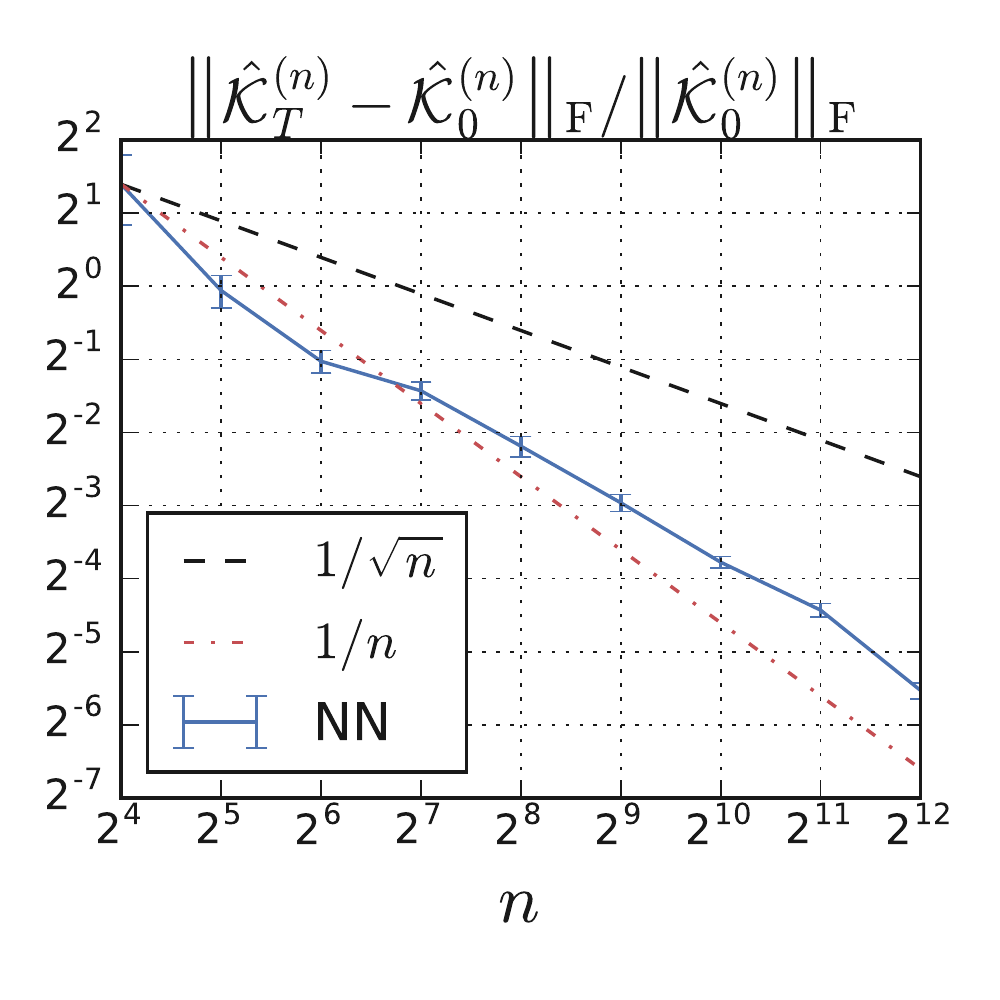} 
  \end{subfigure}
  
  \caption{\textbf{Relative Frobenius norm change during training.} Three hidden layer \relu{} networks trained with $\eta=1.0$ on a subset of MNIST ($|\D|=128$). We measure changes of (input/output/intermediary) weights, empirical $\finntk$, and empirical $\finnngp$ 
  after $T=2^{17}$ steps of gradient descent updates for varying width. We see that the relative change in input/output weights scales as $1/\sqrt{n}$ while intermediate weights scales as $1/n$, this is because the dimension of the input/output does not grow with $n$. The change in $\finntk$ and $\finnngp$ 
  is upper bounded by $\mathcal{O}\left(1/\sqrt{n}\right)$ but is closer to $\mathcal{O}\left(1/n\right)$.
  See Figure \ref{fig_sm:Weight-NTK-vs-width} for the same experiment with 3-layer $\tanh$ and 1-layer \relu{} networks. 
  See Figures \ref{fig:convergence-vs-width-d3} and \ref{fig:convergence-vs-width2} for additional comparisons of finite width empirical and analytic kernels.}
  \label{fig:Weight-NTK-vs-width}
  \vspace{-0.5cm}
\end{figure}

        Note that the updates for individual weights in  \eqref{eq:lin-nn-gradient-descent-weights} vanish in the infinite width limit, which for instance  can be seen from the explicit width dependence of the gradients in the NTK parameterization. Individual weights move by a vanishingly small amount for wide networks in this regime of dynamics, as do hidden layer activations, but they collectively conspire to provide a finite change in the final output of the network, as is necessary for training.
        An additional insight gained from linearization of the network is that the individual instance dynamics derived in \cite{Jacot2018ntk} %
        can be viewed as a random features method,\footnote{We thank Alex Alemi for pointing out a subtlety on correspondence to a random features method.} where the features are the gradients of the model with respect to its weights.

    \subsection{Extensions to other optimizers, architectures, and losses}
    Our theoretical analysis thus far has focused on fully-connected single-output architectures trained by full batch gradient descent. 
    In SM \sref{sec extensions} we derive corresponding results for: networks with multi-dimensional outputs, training against a cross entropy loss, and gradient descent with momentum.

    In addition to these generalizations, there is good reason to suspect the results to extend to much broader class of models and optimization procedures.
    In particular, a wealth of recent literature suggests that the mean field theory governing the wide network limit of fully-connected models~\cite{poole2016exponential,schoenholz2016} extends naturally to residual networks~\cite{yang2017}, CNNs~\cite{xiao18a}, RNNs~\cite{chen2018rnn}, batch normalization~\cite{yang2018a}, and to broad architectures~\cite{yang2019scaling}. 
    We postpone the development of these additional theoretical extensions in favor of an empirical investigation of linearization for a variety of architectures.

\begin{figure}[ht]
  \centering
  \includegraphics[width=\columnwidth]{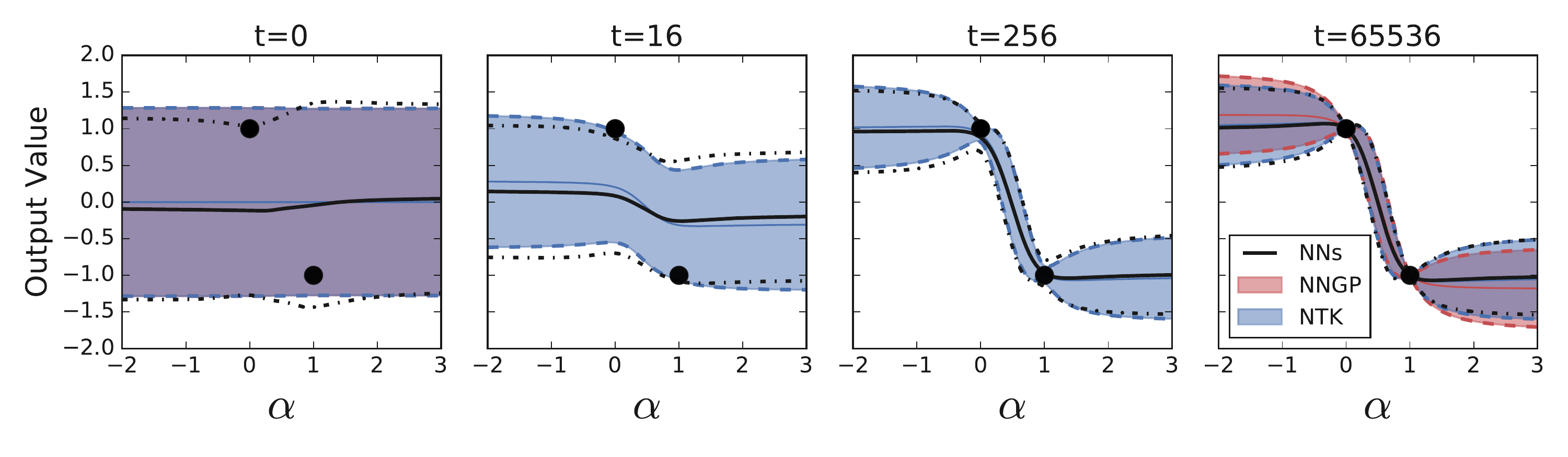} 
  \caption{\textbf{Dynamics of mean and variance of trained neural network outputs follow analytic dynamics from linearization}. Black lines indicate the time evolution of the predictive output distribution from an ensemble of 100 trained neural networks (NNs). The blue region indicates the analytic prediction of the output distribution throughout training (Equations \ref{eq:lin-exact-dynamics-mean}, \ref{eq:lin-exact-dynamics-var}). Finally, the red region indicates the prediction that would result from training only the top layer, corresponding to an NNGP (Equations \ref{eq:NNGP-exact-dynamics-mean}, \ref{eq:NNGP-exact-dynamics-var}).
The trained network has 3 hidden layers of width 8192, $\operatorname{tanh}$ activation functions, $\sigma_w^2 = 1.5$, no bias, and $\eta = 0.5$. The output is computed for inputs interpolated between two training points (denoted with black dots) $x(\alpha) = \alpha x^{(1)} + (1-\alpha)x^{(2)}$. The shaded region and dotted lines denote 2 standard deviations ($\sim 95\%$ quantile) from the mean denoted in solid lines. Training was performed with full-batch gradient descent with dataset size $|\D|=128$. For dynamics for individual function initializations, see SM Figure~\ref{fig:posterior-dynamics-nn-samples}.}
  \label{fig:posterior-dynamics}
  \vspace{-0.5cm}
\end{figure}

\section{Experiments}
\label{sec:experiments}

In this section, we provide empirical support showing that the training dynamics of wide neural networks are well captured by linearized models. We consider fully-connected, convolutional, and wide ResNet architectures trained with full- and mini- batch gradient descent using learning rates sufficiently small so that the continuous time approximation holds well. We consider two-class classification on CIFAR-10 (horses and planes) as well as ten-class classification on MNIST and CIFAR-10. When using MSE loss, we treat the binary classification task as regression with one class regressing to $+1$ and the other to $-1$. 

    Experiments in Figures \ref{fig:Weight-NTK-vs-width}, \ref{fig:NTK-dynamics-wresnet}, \ref{fig:NTK-dynamics-cnn}, \ref{fig:NTK-dynamics-xent-mom}, \ref{fig:logit-deviation-xent}, \ref{fig:error-vs-width} and \ref{fig_sm:Weight-NTK-vs-width}, were done in JAX \citep{jaxrepo}. The remaining experiments used TensorFlow \citep{abadi2016tensorflow}.
An open source implementation of this work providing tools to investigate linearized learning dynamics is available at \href{https://www.github.com/google/neural-tangents}{\texttt{\textbf www.github.com/google/neural-tangents}} ~\cite{neuraltangents2019}.

\begin{figure}%
    \centering
  \includegraphics[width=.85\columnwidth]{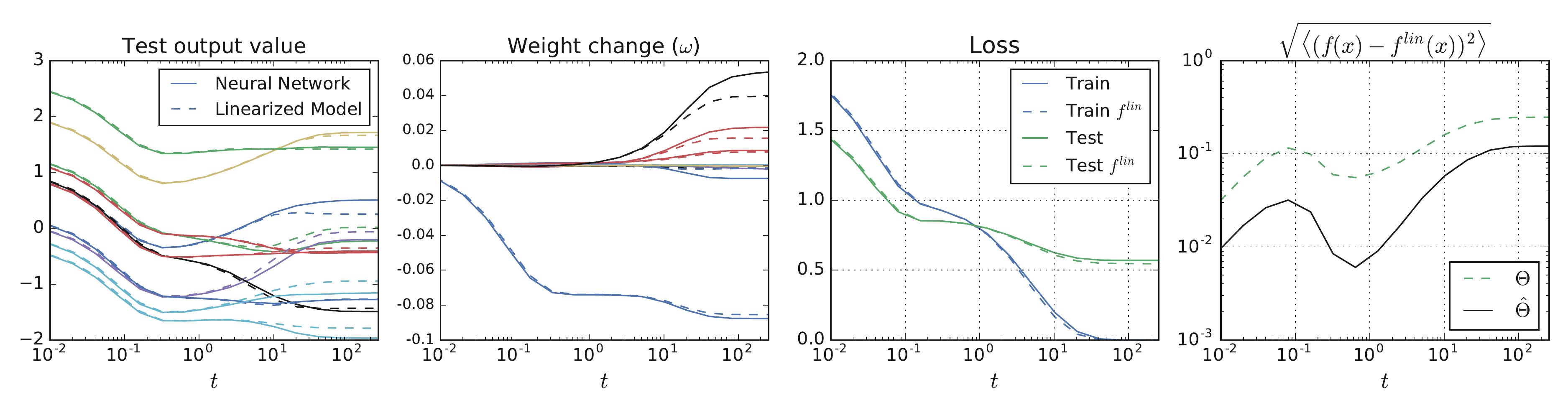} 
  \caption{{\bf Full batch gradient descent on a model behaves similarly to analytic dynamics on its linearization, both for network outputs, and also for individual weights.} 
  A binary CIFAR classification task with MSE loss and a $\operatorname{ReLU}$ fully-connected network with 5 hidden layers of width $n=2048$, $\eta = 0.01$, $|\D|=256$, $k=1$, $\sigma_w^2=2.0$, and $\sigma_b^2=0.1$. Left two panes show dynamics for a randomly selected subset of datapoints or parameters. Third pane shows that the dynamics of loss for training and test points agree well between the original and linearized model. The last pane shows the dynamics of RMSE between the two models on test points. We observe that the empirical kernel $\finntk$ gives more accurate dynamics for finite width networks.
  }
  \label{fig:NTK-dynamics}
\end{figure}

\textbf{Predictive output distribution}:
In the case of an MSE loss, the output distribution remains Gaussian throughout training.
In Figure~\ref{fig:posterior-dynamics}, the predictive output distribution for input points interpolated between two training points is shown for an ensemble of neural networks and their corresponding GPs. The interpolation is given by $x(\alpha) = \alpha x^{(1)} + (1-\alpha)x^{(2)}$ where $x^{(1,2)}$ are two training inputs with different classes. 
We observe that the mean and variance dynamics of neural network outputs during gradient descent training follow the analytic dynamics from linearization well (Equations \ref{eq:lin-exact-dynamics-mean}, \ref{eq:lin-exact-dynamics-var}). 
Moreover the NNGP predictive distribution which corresponds to exact Bayesian inference, while similar, is noticeably \emph{different} from the predictive distribution at the end of gradient descent training. For dynamics for individual function draws see SM Figure~\ref{fig:posterior-dynamics-nn-samples}.

\textbf{Comparison of training dynamics of linearized network to original network}:
For a particular realization of a finite width network, one can analytically predict the dynamics of the weights and outputs over the course of training using the empirical tangent kernel at initialization.  In Figures ~\ref{fig:NTK-dynamics}, \ref{fig:NTK-dynamics-wresnet} (see also \ref{fig:NTK-dynamics-cnn}, \ref{fig:NTK-dynamics-xent-mom}), we compare these linearized dynamics (Equations \ref{eq:lin-dynamics-weights},~\ref{eq:lin-dynamics-outputs}) with the result of training the actual network. In all cases we see remarkably good agreement. We also observe that for finite networks, dynamics predicted using the empirical kernel $\finntk$ better match the data than those obtained using the infinite-width, analytic, kernel $\infntk$. To understand this we note that $\|\finntk^{(n)}_T -\finntk^{(n)}_0\|_F = \mathcal O(1 /n) \leq \mathcal O( 1/ {\sqrt n})=\|\finntk^{(n)}_0 - \infntk\|_F$,
where $\finntk^{(n)}_0$ denotes the empirical tangent kernel of width $n$ network, as plotted in Figure~\ref{fig:Weight-NTK-vs-width}.

One can directly optimize parameters of $\flin$ instead of solving the ODE induced by the tangent kernel $\finntk$. Standard neural network optimization techniques such as mini-batching, weight decay, and data augmentation can be directly applied. In Figure~\ref{fig:NTK-dynamics-wresnet} (\ref{fig:NTK-dynamics-cnn}, \ref{fig:NTK-dynamics-xent-mom}), we compared the training dynamics of the linearized and original network while directly training both networks.

With direct optimization of linearized model, we tested full ($|\D|= 50,000$) MNIST digit classification with cross-entropy loss, and trained with a momentum optimizer (Figure~\ref{fig:NTK-dynamics-xent-mom}). 
For cross-entropy loss with softmax output, some logits at late times grow indefinitely, in contrast to MSE loss where logits converge to target value. The error between original and linearized model for cross entropy loss becomes much worse at late times if the two models deviate significantly before the logits enter their late-time steady-growth regime (See Figure~\ref{fig:logit-deviation-xent}).

Linearized dynamics successfully describes the training of networks beyond vanilla fully-connected models. 
To demonstrate the generality of this procedure
we show we can predict the learning dynamics of subclass of Wide Residual Networks (WRNs)~\cite{zagoruyko2016wide}.
WRNs are a class of model that are popular in computer vision and leverage convolutions, batch normalization, skip connections, and average pooling. In Figure~\ref{fig:NTK-dynamics-wresnet}, we show a comparison between the linearized dynamics and the true dynamics for a wide residual network trained with MSE loss and SGD with momentum, \emph{trained on the full CIFAR-10 dataset}. We slightly modified the block structure described in Table~\ref{tab:wide_resnet_config} so that each layer has a constant number of channels (1024 in this case), and otherwise followed the original implementation. 
As elsewhere, we see strong agreement between the predicted dynamics and the result of training.

\begin{figure}[t]
  \centering
  \includegraphics[width=\columnwidth]{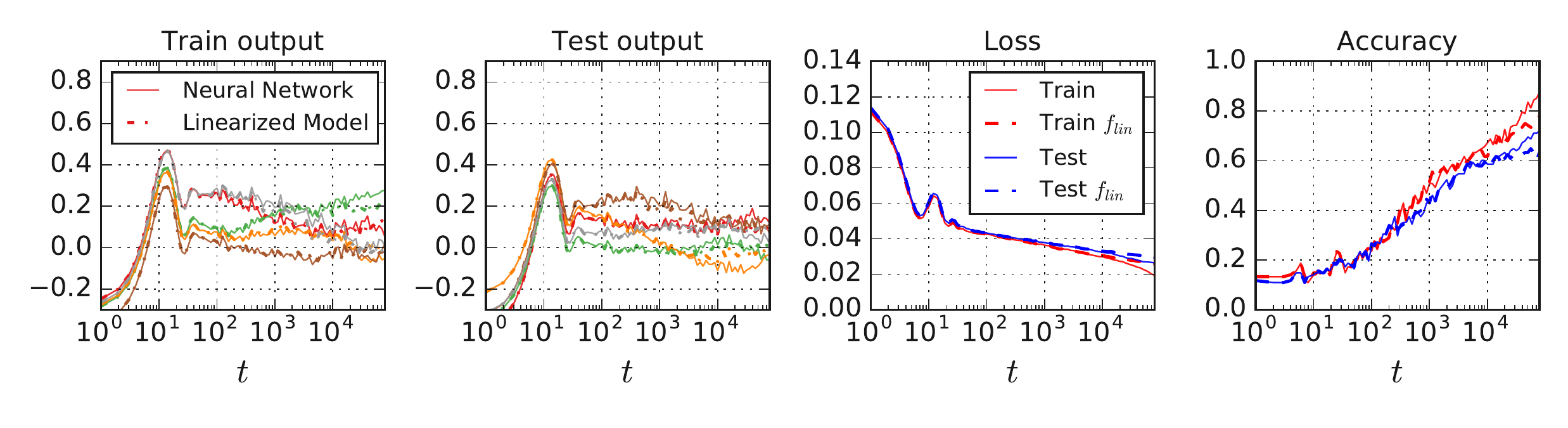}
  \caption{{\bf A wide residual network and its linearization behave similarly when both are trained by SGD with momentum on MSE loss on CIFAR-10.}
  We adopt the network architecture from~\citet{zagoruyko2016wide}. We use $N=1$, channel size $1024$, $\eta = 1.0$, $\beta=0.9$, $k=10$, $\sigma_w^2=1.0$, and $\sigma_b^2=0.0$. See Table~\ref{tab:wide_resnet_config} for details of the architecture. Both the linearized and original model are trained directly on full CIFAR-10 ($|\D|= 50,000$), using SGD with batch size 8. 
  Output dynamics for a randomly selected subset of train and test points are shown in the first two panes. 
  Last two panes show training and accuracy curves for the original and linearized networks.
 }
  \label{fig:NTK-dynamics-wresnet}
  \vspace{-0.5cm}
\end{figure}

\textbf{Effects of dataset size}:
The training dynamics of a neural network match those of its linearization when the width is infinite and the dataset is finite. In previous experiments, we chose sufficiently wide networks to achieve small error between neural networks and their linearization for smaller datasets. 
Overall, we observe that as the width grows the error decreases (Figure~\ref{fig:error-vs-width}). 
Additionally, we see that the error grows 
in the size of the dataset. Thus, although error grows with dataset this can be counterbalanced by a corresponding increase in the model size.

\section{Discussion}

We showed theoretically that the learning dynamics in parameter space of deep nonlinear neural networks are exactly described by a linearized model in the infinite width limit. Empirical investigation 
revealed that this agrees well with actual training dynamics and predictive distributions across fully-connected, convolutional, and even wide residual network architectures, as well as with different optimizers (gradient descent, momentum, mini-batching) and loss functions (MSE, cross-entropy). Our results suggest that a surprising number of realistic neural networks may be operating in the regime we studied. 
This is further consistent with recent experimental work showing that neural networks are often robust to re-initialization but not re-randomization of layers (\citet{zhang2019all}).

In the regime we study, since the learning dynamics are fully captured by the kernel $\finntk$ and the target signal, studying the properties of $\finntk$ to determine trainability and generalization are interesting future directions. Furthermore, the infinite width limit gives us a simple characterization of both gradient descent and Bayesian inference. 
By studying properties of the NNGP kernel $\infnngp$ and the tangent kernel $\infntk$, we may shed light on the inductive bias of gradient descent. 

Some layers of modern neural networks may be operating far from the linearized regime. Preliminary observations in~\citet{lee2018deep} showed that wide neural networks trained with SGD perform similarly to the corresponding GPs as width increase, while GPs still outperform trained neural networks for both small and large dataset size.
Furthermore, in~\citet{novak2018bayesian}, it is shown that the comparison of performance between finite- and infinite-width networks is highly architecture-dependent. In particular, it was found that infinite-width networks perform as well as or better than their finite-width counterparts for many fully-connected or locally-connected architectures. 
However, the opposite was found in the case of convolutional networks without 
pooling. 
It is still an open research question to determine the main factors that determine these performance gaps. We believe that examining the behavior of infinitely wide networks provides a strong basis from which to build up a systematic understanding of finite-width networks (and/or networks trained with large learning rates).

\section*{Acknowledgements}
We thank Greg Yang and Alex Alemi for useful discussions and feedback. We are grateful to Daniel Freeman, Alex Irpan and anonymous reviewers for providing valuable feedbacks on the draft. We thank the JAX team for developing a language which makes model linearization and NTK computation straightforward. We would like to especially thank Matthew Johnson for support and debugging help.

\bibliography{references}
\bibliographystyle{unsrtnat}

\normalsize
\onecolumn
\clearpage
\appendix

\begin{center}
\textbf{\large Supplementary Material}
\end{center}
\setcounter{equation}{0}
\setcounter{figure}{0}
\setcounter{table}{0}
\setcounter{page}{1}
\setcounter{section}{0}
\makeatletter
\renewcommand{\theequation}{S\arabic{equation}}
\renewcommand{\thefigure}{S\arabic{figure}}
\renewcommand{\thetable}{S\arabic{table}}

\section{Additional figures}

\begin{figure}[h!]
  \centering
  \includegraphics[width=\columnwidth]{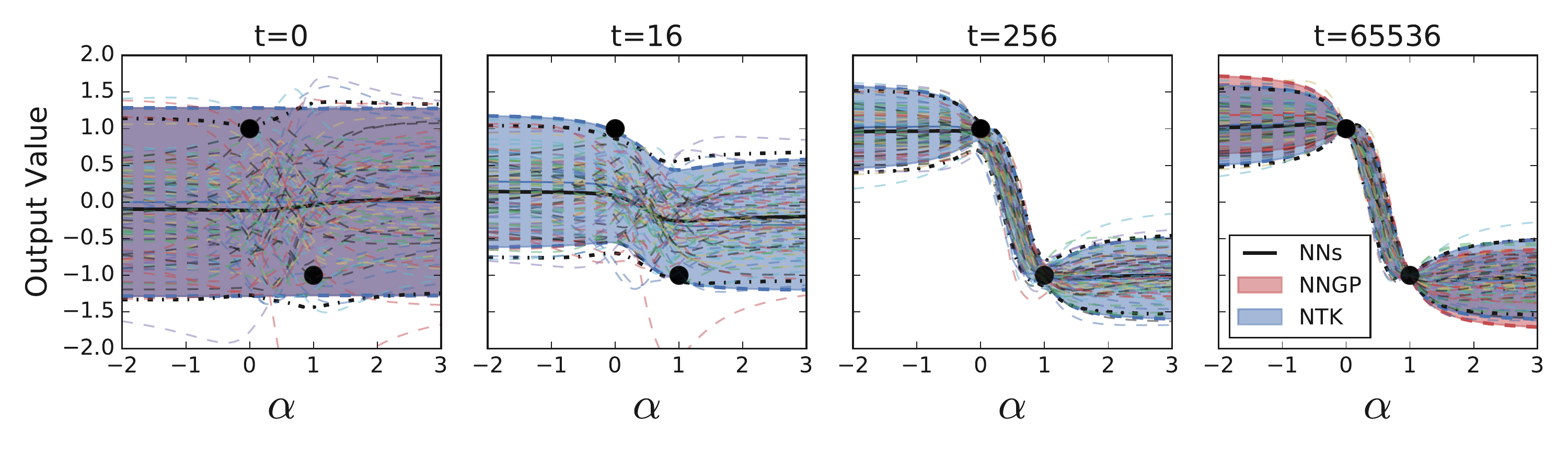} 
  \caption{\textbf{Sample of neural network outputs.} The lines correspond to the functions learned for 100 different initializations. The configuration is the same as in Figure~\ref{fig:posterior-dynamics}.}
  \label{fig:posterior-dynamics-nn-samples}
\end{figure}

\begin{figure}[h!]
  \centering
  \includegraphics[width=\columnwidth]{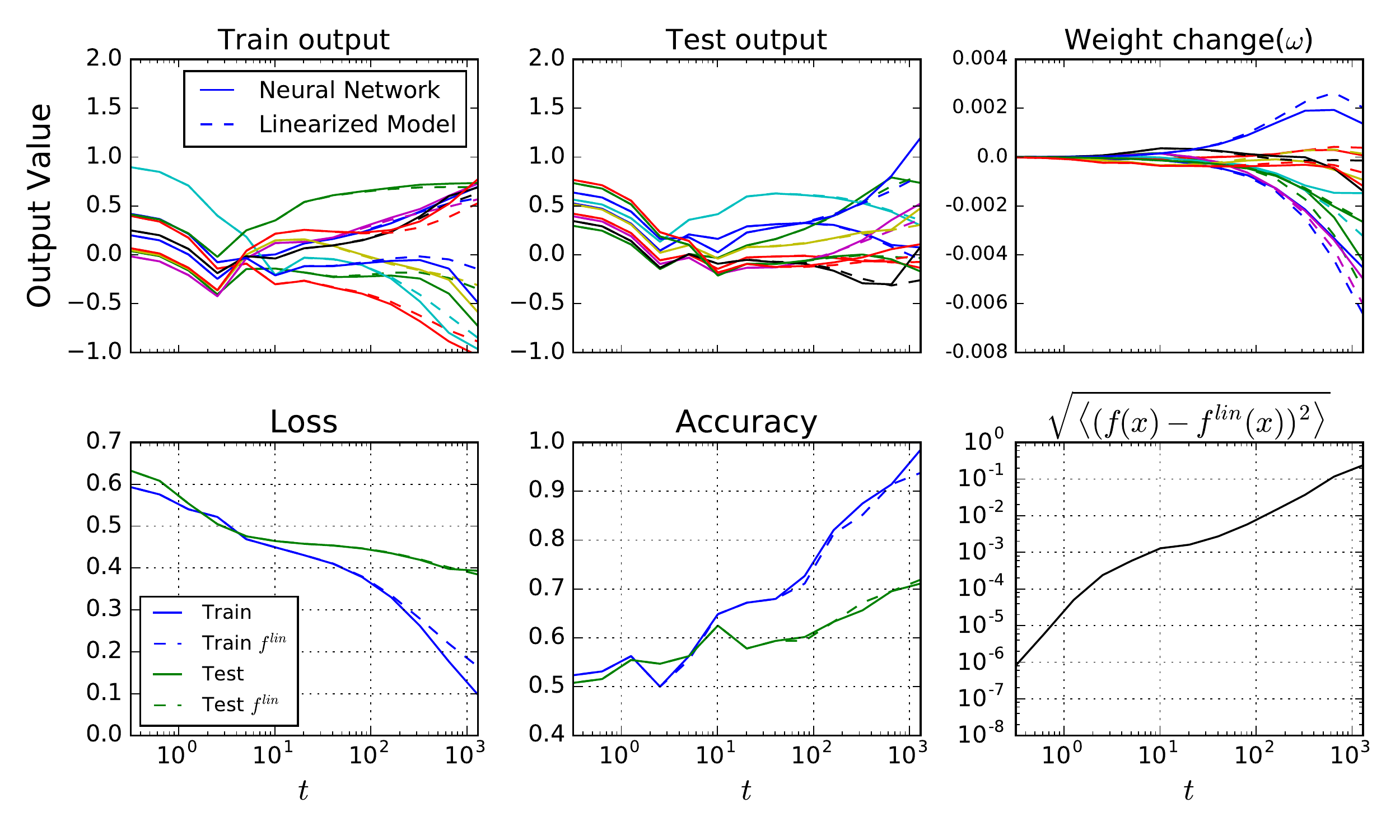}
  \caption{{\bf 
  A convolutional network and its linearization behave similarly when trained using full batch gradient descent with a momentum optimizer.}
  Binary CIFAR classification task with MSE loss, $\operatorname{tanh}$ convolutional network with 3 hidden layers of channel size $n=512$, $3\times3$ size filters, average pooling after last convolutional layer, $\eta = 0.1$, $\beta=0.9$, $|\D|= 128$, $\sigma_w^2=2.0$ and $\sigma_b^2=0.1$. 
  The linearized model is trained directly by full batch gradient descent with momentum, rather than by integrating its continuous time analytic dynamics.
  Panes are the same as in Figure \ref{fig:NTK-dynamics}.
}
  \label{fig:NTK-dynamics-cnn}
\end{figure}

\begin{figure}[h!]
  \centering
  \includegraphics[width=\columnwidth]{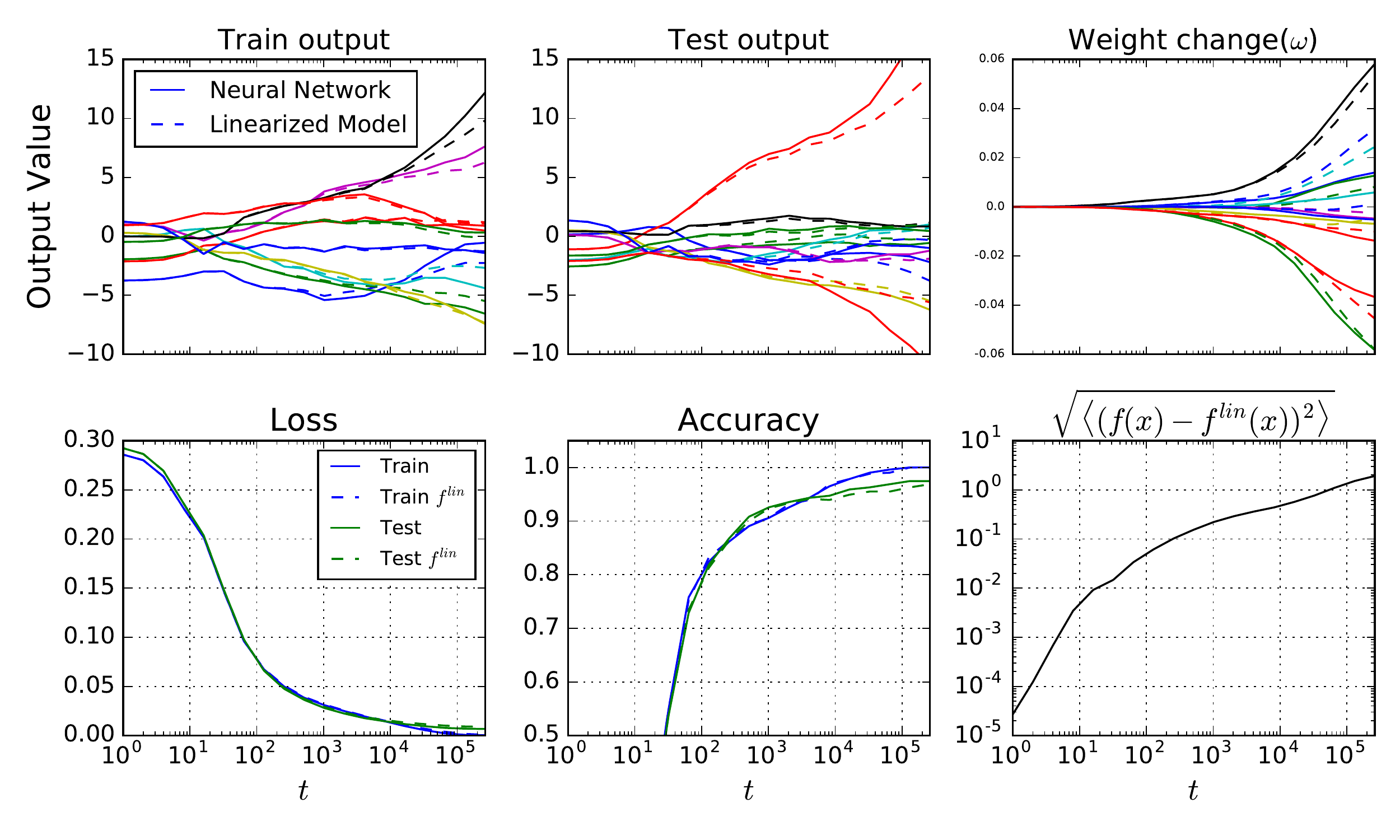}
  \caption{{\bf A neural network and its linearization behave similarly when both are trained
  via SGD with momentum 
  on cross entropy loss on MNIST.}
  Experiment is for 10 class MNIST classification using a $\operatorname{ReLU}$ fully connected network with 2 hidden layers of width $n=2048$, $\eta = 1.0$, $\beta=0.9$, $|\D|=50,000$, $k=10$, $\sigma_w^2=2.0$, and $\sigma_b^2=0.1$. Both models are trained using stochastic minibatching with batch size 64. 
  Panes are the same as in Figure \ref{fig:NTK-dynamics}, except that the top row shows all ten logits for a single randomly selected datapoint.
  }
  \label{fig:NTK-dynamics-xent-mom}
\end{figure}

\begin{figure}[h!]
  \centering
  \includegraphics[width=\columnwidth]{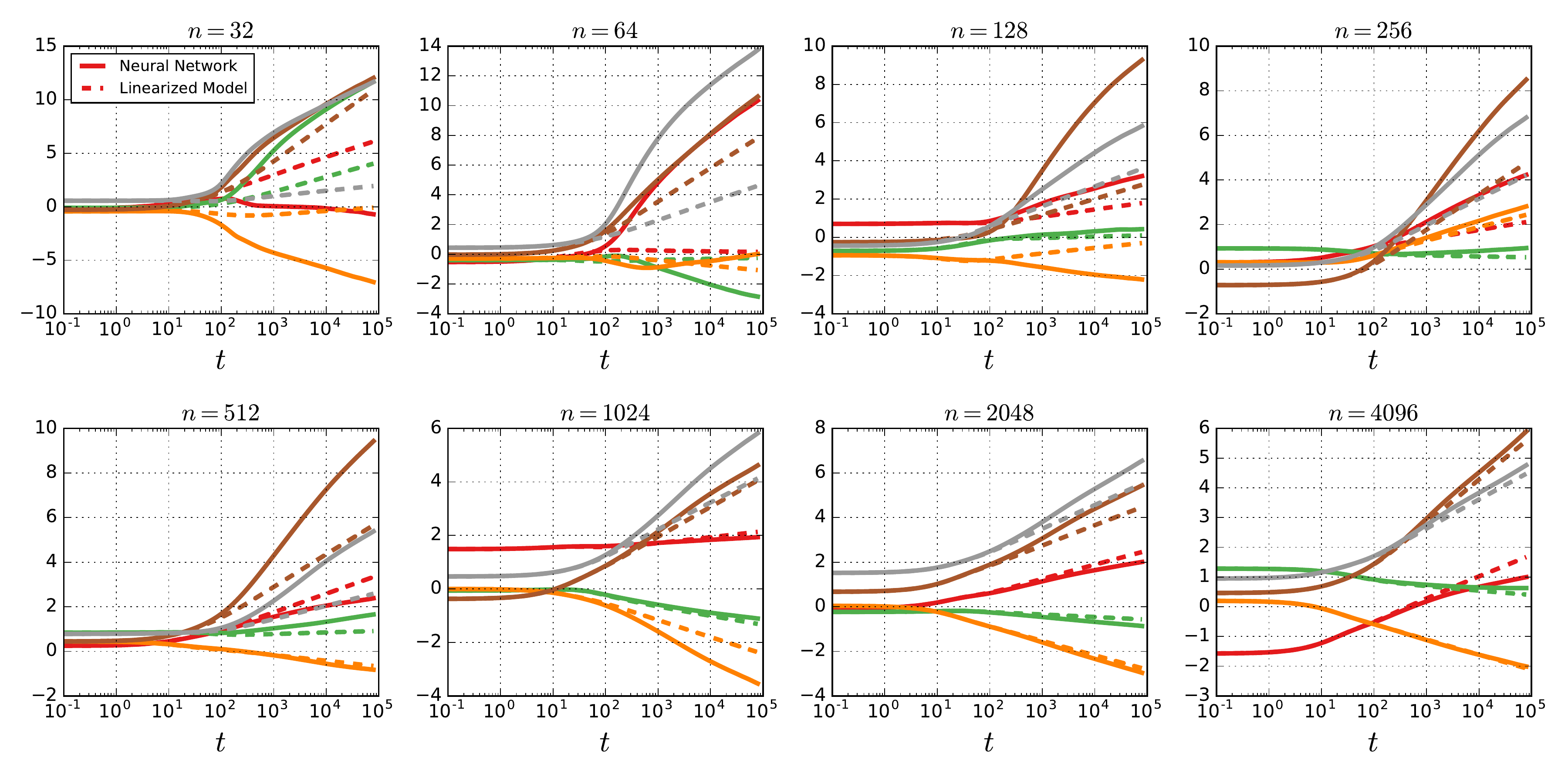} 
  \caption{\textbf{Logit deviation for cross entropy loss.} Logits for models trained with cross entropy loss diverge at late times. If the deviation between the logits of the linearized model and original model are large early in training, as shown for the narrower networks (first row), logit deviation at late times can be significantly large. As the network becomes wider (second row), the logit deviates at a later point in training. Fully connected $\operatorname{tanh}$ network $L=4$ trained on binary CIFAR classification problem.  }
  \label{fig:logit-deviation-xent}
\end{figure}

\begin{figure}[h!]
  \centering
  \includegraphics[width=\columnwidth]{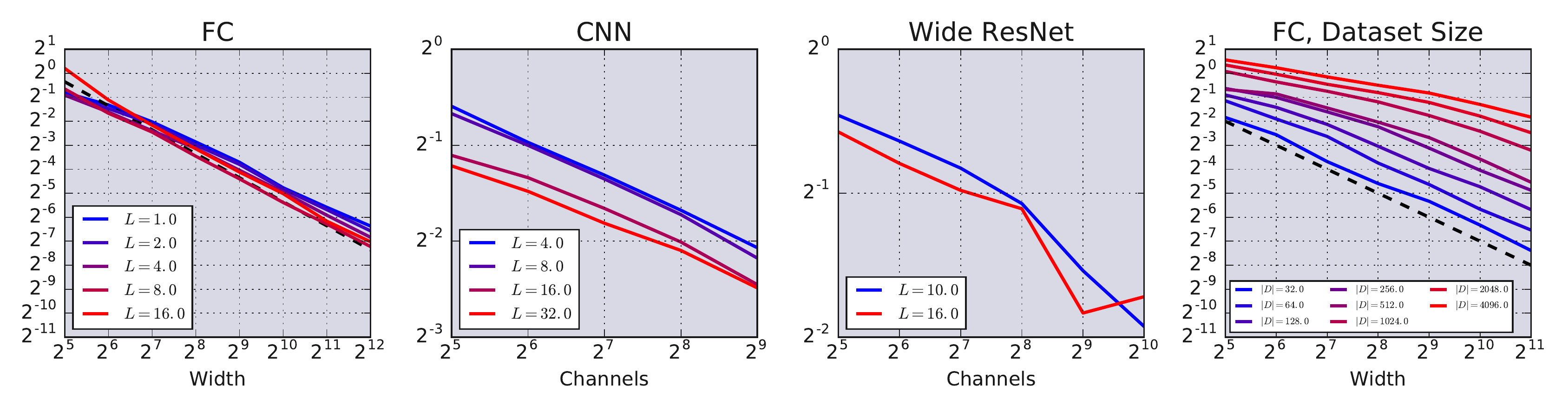}
  \caption{\textbf{Error dependence on depth, width, and dataset size.} Final value of the RMSE for fully-connected, convolutional, wide residual network as networks become wider for varying depth and dataset size. 
  Error in fully connected networks as the depth is varied from 1 to 16 (first) and the dataset size is varied from 32 to 4096 (last). Error in convolutional networks as the depth is varied between 1 and 32 (second), and WRN for depths 10 and 16 corresponding to N=1,2 described in Table~\ref{tab:wide_resnet_config} (third). Networks are critically initialized $\sigma_w^2=2.0$, $\sigma_b^2=0.1$, trained with gradient descent on MSE loss. Experiments in the first three panes used $|\D|=128$.}
  \label{fig:error-vs-width}
\end{figure}

 \begin{figure}[h!]
  \centering
  \begin{subfigure}[b]{0.24\columnwidth}
  \includegraphics[width=\textwidth]{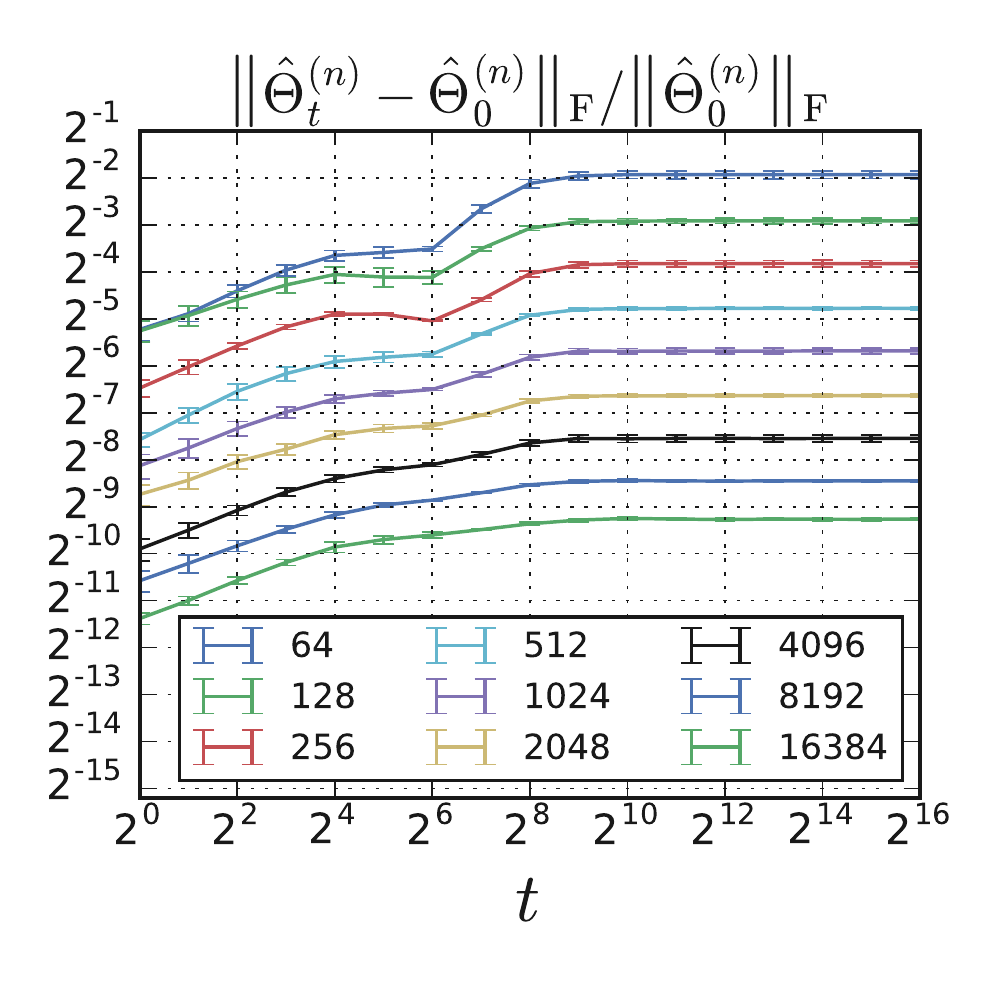} 
  \end{subfigure}
  \begin{subfigure}[b]{0.24\columnwidth}
  \includegraphics[width=\textwidth]{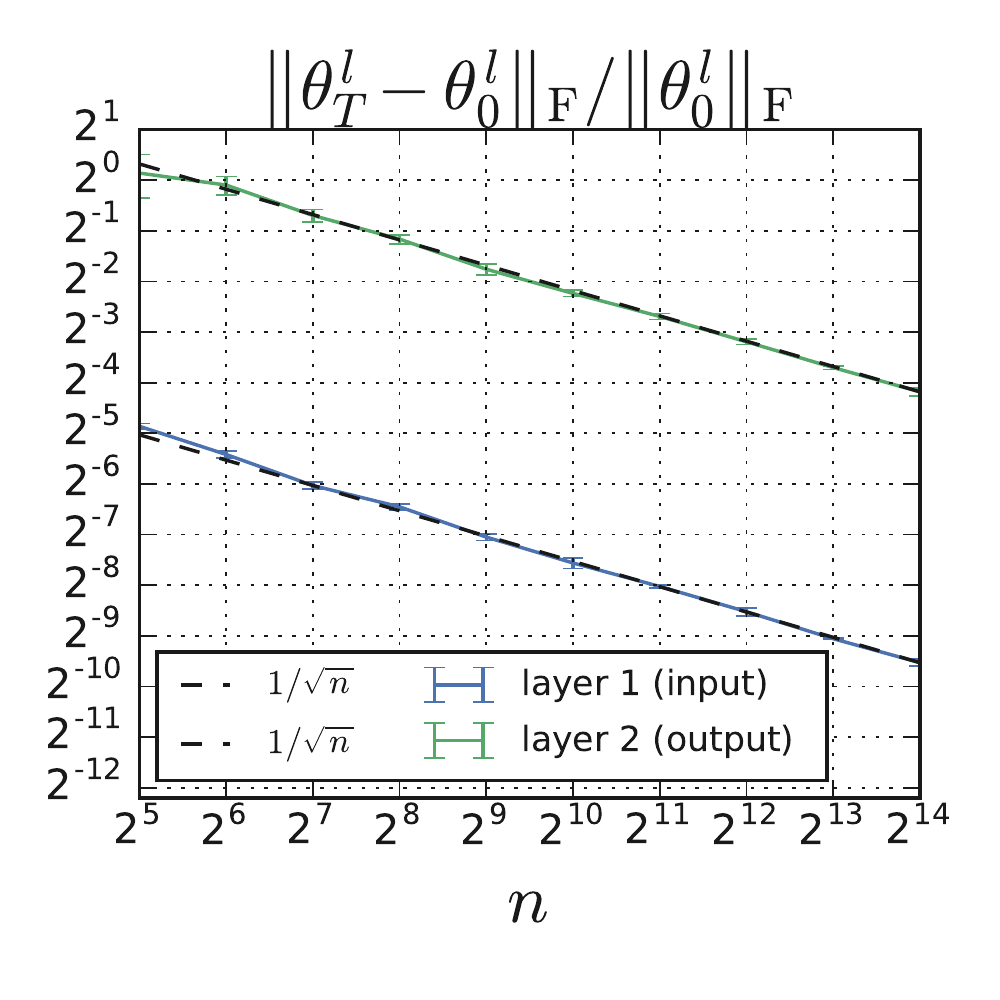} 
  \end{subfigure}
  \begin{subfigure}[b]{0.24\columnwidth}
  \includegraphics[width=\textwidth]{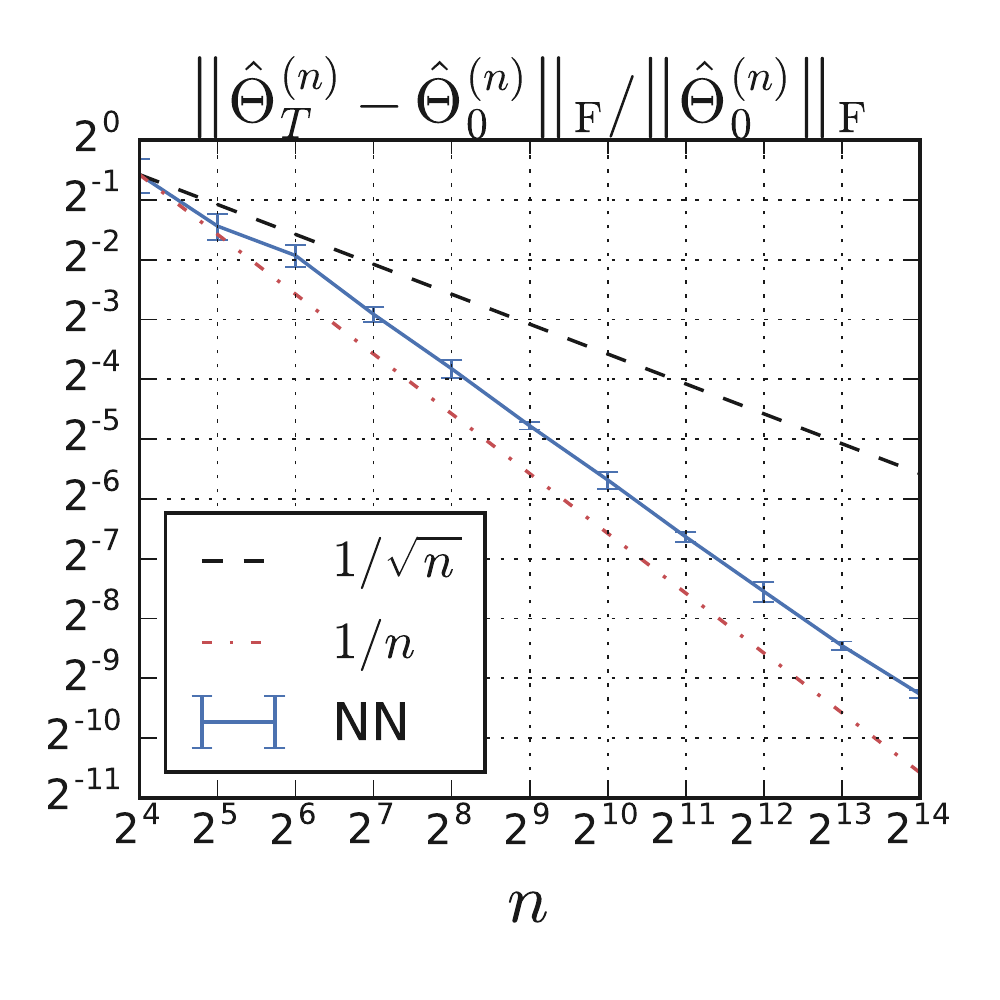} 
  \end{subfigure}
  \begin{subfigure}[b]{0.24\columnwidth}
  \includegraphics[width=\textwidth]{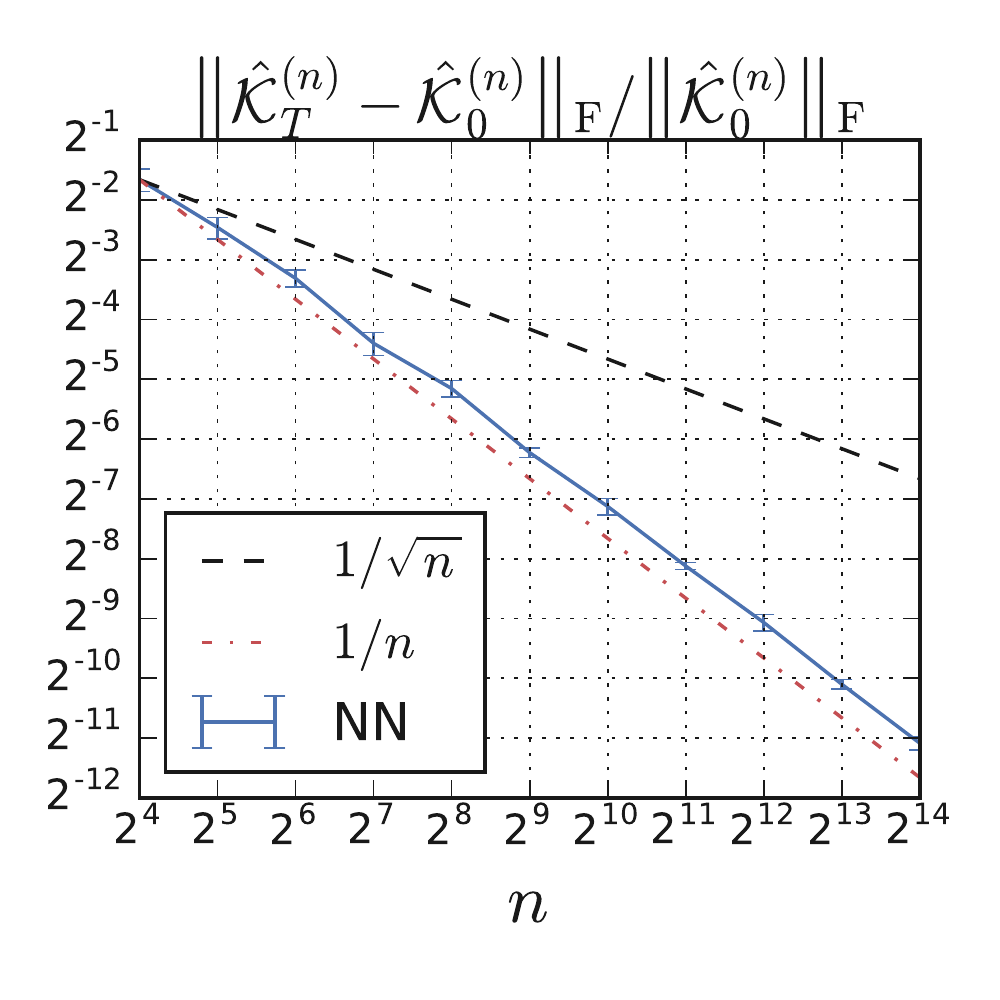} 
  \end{subfigure}

  \begin{subfigure}[b]{0.24\columnwidth}
  \includegraphics[width=\textwidth]{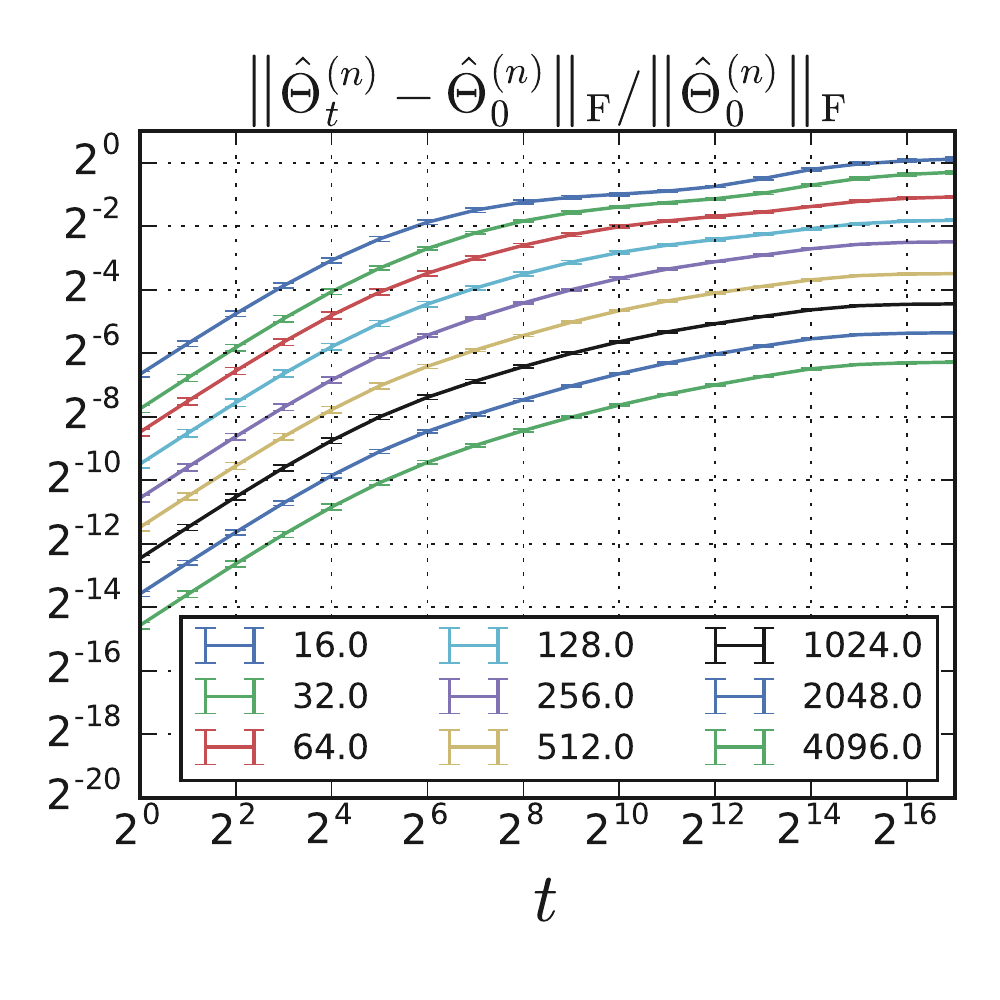} 
  \end{subfigure}
  \begin{subfigure}[b]{0.24\columnwidth}
  \includegraphics[width=\textwidth]{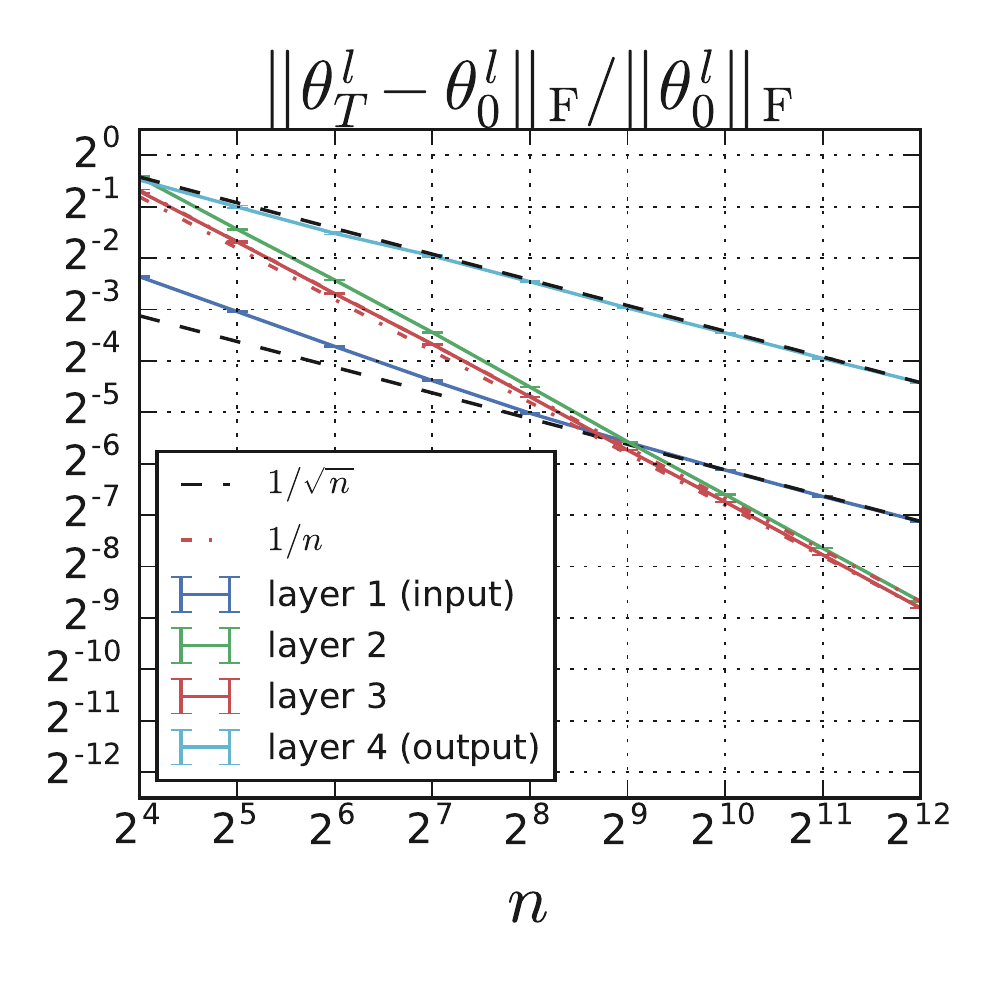} 
  \end{subfigure}
  \begin{subfigure}[b]{0.24\columnwidth}
  \includegraphics[width=\textwidth]{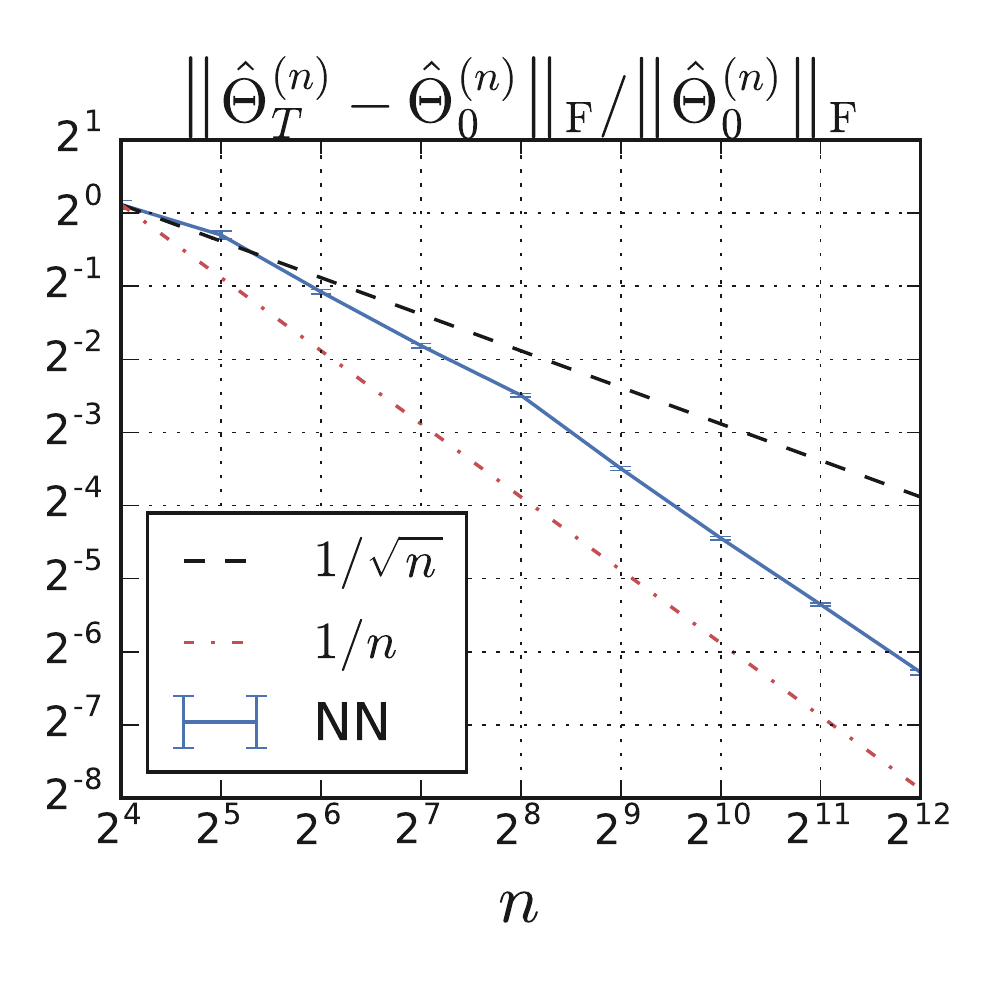} 
  \end{subfigure}
  \begin{subfigure}[b]{0.24\columnwidth}
  \includegraphics[width=\textwidth]{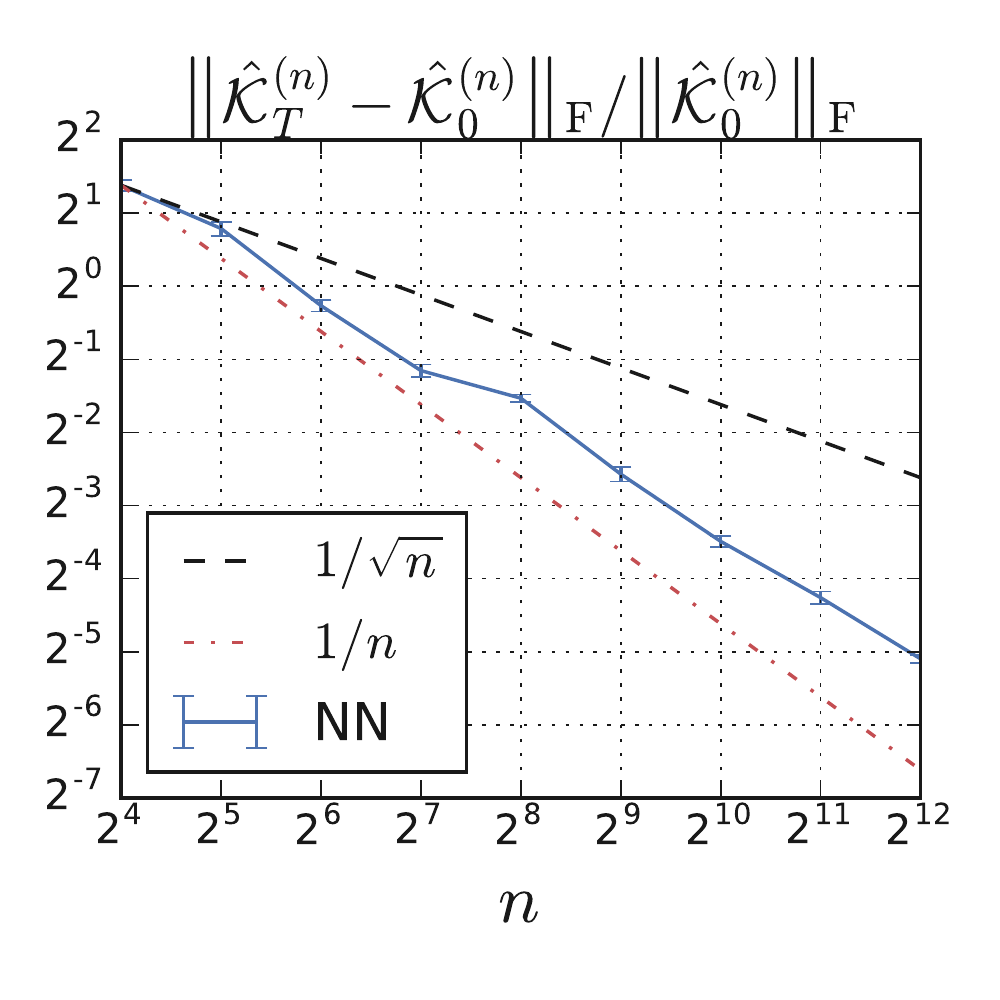} 
  \end{subfigure}  
  
  \caption{\textbf{Relative Frobenius norm change during training.} \emph{(top)} One hidden layer, \relu{} networks trained with $\eta=1.0$, on a 2-class CIFAR10 subset of size $|\D|=128$. We measure changes of (read-out/non read-out) weights, empirical $\finntk$ and empirical $\finnngp$ 
  after $T=2^{16}$ steps of gradient descent updates for varying width. 
  \emph{(bottom)} Networks with three layer $\tanh$ nonlinearity and other details are identical to Figure \ref{fig:Weight-NTK-vs-width}.}
  \label{fig_sm:Weight-NTK-vs-width}
\end{figure}

\newpage
\newpage

\section{Extensions}\label{sec extensions}
\subsection{Momentum}
    One direction is to go beyond vanilla gradient descent dynamics. We 
    consider momentum updates\footnote{Combining the usual two stage update into a single equation.}
    \begin{align}
     \theta_{i+1} = \theta_i + \beta(\theta_i - \theta_{i-1}) - \eta \nabla_\theta \mathcal L|_{\theta = \theta_i} \,.
    \end{align}
    The discrete update to the function output becomes
    \begin{align}
        \flin_{i+1}(x)
        =\flin_{i}(x)
        -\eta \finntk_0(x, \X)\nabla_{\flin_i(\mathcal{X})} \mc L + \beta (\flin_{i}(x) - \flin_{i-1}(x)) 
    \end{align}
    where $\flin_{t}(x)$ is the output of the linearized network after $t$ steps.
    One can take the continuous time limit as in \citet{qian1999momentum,su2014differential} and obtain
    \begin{align}
        \ddot \omega_t &= \tilde \beta \dot \omega_t - \nabla_\theta \flin_0(\X)^T \nabla_{\flin_t(\mathcal{X})} \mc L \\ 
        \ddot f_t {} ^{\textrm{lin}}(x) &= \tilde \beta \dot f_t ^{\textrm{lin}}(x) -  \finntk_0(x, \X) \nabla_{\flin_t(\mathcal{X})} \mc L
        \label{eq:lin-mom-output}
    \end{align}
    where continuous time relates to steps $t=i \sqrt{\eta}$ and $\tilde \beta = (\beta - 1)/\sqrt{\eta}$. These equations are also amenable to analytic treatment for MSE loss. See Figure~\ref{fig:NTK-dynamics-cnn},~\ref{fig:NTK-dynamics-xent-mom} and~\ref{fig:NTK-dynamics-wresnet} for experimental agreement.

    \subsection{Multi-dimensional output and cross-entropy loss}
    \label{sec:CrossEntropy}
    One can extend the loss function to general functions with multiple output dimensions.
    Unlike for squared error, we do not have a closed form solution to the dynamics equation. However, the equations for the dynamics can be solved using an ODE solver as an initial value problem. 
    \begin{equation}
    \ell(f,  y) = - \sum_i  y^i \log \sigma(f^i), \qquad \sigma(f^i) \equiv \frac{\exp(f^i)}{\sum_j \exp(f^j)}\,.
    \end{equation}
    Recall that $\frac{\partial \ell}{\partial \hat y^i} = \sigma(\hat y^i) -  y^i$.
    For general input point $x$ and for an arbitrary parameterized function $f ^i (x)$ parameterized by $\theta$, gradient flow dynamics is given by 
    \begin{align}
    \dot  {f}_t^i(x)= \nabla_\theta  f_t^i(x)   \frac {d\theta} {dt} &= -  \eta \nabla_\theta  f_t^i(x)    \sum_j\sum_{(z, y)\in\D} \left[ \nabla_\theta f_t^j(z)^T\frac {\partial \ell(f_t, y) }{\partial \hat y^j}  \right] \\
    &= -  \eta  \sum_{(z, y)\in\D}\sum_j \nabla_\theta f_t^i(x) \nabla_\theta f_t^j(z) ^T \left(\sigma(f_t^j(z)) - y^j\right)
    \end{align}
    Let $\finntk^{ij}(x, \X) = \nabla_\theta f^i(x) \nabla_\theta f^j(\X) ^T$. The above is   
    \begin{align}
    \dot {f_t} (\X) &= -\eta  \finntk_t (\X, \X) \left( \sigma(f_t(\X)) - \Y \right)\\
    \dot {f_t} (x) &= -  \eta \finntk_t (x, \X) \left( \sigma(f_t(\X)) - \Y \right)\,.
    \end{align}
    The linearization is 
    \begin{align}
    \label{eq:xent-ode-train}
    \dot {f_t}^{\textrm{lin}} (\X) &= - \eta  \finntk_0 (\X, \X) \left( \sigma(f^{\textrm{lin}}_t(\X)) - \Y \right)\\
    \label{eq:xent-ode-test}
    \dot {f_t}^{\textrm{lin}} (x) &= -  \eta \finntk_0 (x, \X) \left( \sigma(f_t^{\textrm{lin}}(\X)) - \Y \right)\,.
    \end{align}
    For general loss, e.g. cross-entropy with softmax output, we need to rely on solving the ODE Equations \ref{eq:xent-ode-train} and \ref{eq:xent-ode-test}. We use the \texttt{dopri5} method for ODE integration, which is the default integrator in TensorFlow (\texttt{tf.contrib.integrate.odeint}).

\section{Neural Tangent kernel for \relu{} and \erf{}}
\label{sec:analytic_kernel}
For \relu{} and \erf{} activation functions, the tangent kernel can be computed analytically. We begin with the case $\phi = $ \relu{}; using the formula from \citet{cho2009}, we can compute $\T$ and $\dot\T$ in closed form. 
Let $\Sigma $ be a $2\times 2$ PSD matrix.
We will use
\begin{align}
    k_n(x , y)  = \int \phi^n(x \cdot w) \phi^n(y \cdot w) e^{-\|w\|^2/2} dw \cdot (2\pi)^{-d/2}  = \frac 1  {2\pi} \|x\|^{n} \|y\|^{n}  J_n(\theta)    
\end{align}
where 
\begin{align}
    \phi(x) &= \max(x, 0), \quad \theta(x, y) = \arccos \left(\frac{x\cdot y} {\|x\|\|y\|} \right)\,,
    \nonumber \\
    J_0(\theta) &= \pi - \theta\,, \quad
    J_1(\theta) = \sin \theta + (\pi - \theta) \cos \theta
    =  \sqrt{1 - \left(\frac{x\cdot y} {\|x\|\|y\|} \right)^2 } + (\pi - \theta)   \left(\frac{x\cdot y} {\|x\|\|y\|} \right)\,.
\end{align}
Let $d=2$ and $u = (x\cdot w, y\cdot w)^T$. Then $u$ is a mean zero Gaussian with $\Sigma = [[x\cdot x, x\cdot y]; [x\cdot y, y\cdot y]]$. 
Then  
\begin{align}
    \T(\Sigma)   &= k_1(x, y) =   \frac 1  {2\pi} \|x\| \|y\|  J_1(\theta) \\
    \dot\T(\Sigma) &= k_0(x, y) =  \frac 1  {2\pi}   J_0(\theta) 
\end{align}

For $\phi= \operatorname{erf}$, let $\Sigma$ be the same as above. Following \citet{williams1997}, we get
\begin{align}
    \T(\Sigma) &= \frac 2 \pi \sin^{-1} \left(\frac {2x\cdot y} {\sqrt {(1 + 2 x \cdot x) (1 + 2 y \cdot y)} }\right)
    \\
    \dot\T(\Sigma) &= \frac 4 \pi  {\rm det} (I + 2\Sigma)^{-1/2}
\end{align}

\section{Gradient flow dynamics for training only the readout-layer}
\label{sec:gradient-readout-layer}
The connection between Gaussian processes and Bayesian wide neural networks can be extended to the setting when only the readout layer parameters are being optimized.  
More precisely, we show that when training only the readout layer, the outputs of the network form a Gaussian process (over an ensemble of draws from the parameter prior) throughout training, where that output is an interpolation between the GP prior and GP posterior.

Note that for any $x, x'\in\mathbb R^{n_0}$, in the infinite width limit   
$\bar x (x) \cdot \bar x (x') =\finnngp(x, x') \to \infnngp(x, x')$ in probability, 
where for notational simplicity we assign $\bar {x}(x) = \left[\frac{\sigma_w x^{L}(x)}{\sqrt {n_L}}, {\sigma_b}\right]$.
The regression problem is specified with mean-squared loss
\begin{align}
    \mathcal L =  \frac 1 2   \|f(\X) - \Y\|_2^2
    =\frac 1 2 \|{\bar x}(\X) \theta^{L+1}  - \Y\|_2^2, 
\end{align}
and applying gradient flow to optimize the readout layer (and freezing all other parameters),  
\begin{align}
    \dot \theta^{L+1}  = - {\eta}   {\bar x(\X)}^T \left( {\bar x}(\X) \theta^{L+1}  - \Y \right)\, ,
\end{align}
where $\eta$ is the learning rate. The solution to this ODE gives the evolution of the output of an arbitrary $x^*$. 
So long as the empirical kernel $\bar x(\X)\bar x(\X)^T$ is invertible, it is
\begin{align}
\label{eq:NNGP-exact-dynamics}
f_t(x^*)  =
f_0(x^*)+ \finnngp(x, \X)\finnngp(\X, \X)^{-1}
\left(\exp\left(-\eta t \finnngp(\X, \X )\right)-I \right)(f_0(\X) - \Y)  
\end{align}
For any $x, x'\in\mathbb R^{n_0}$, letting $n_l\to \infty$ for $l=1, \dots, L$, one has the convergence in distribution in probability and distribution respectively
\begin{align}
    \bar x(x) \bar x(x') \to \infnngp(x, x') 
    \quad \textrm{and}\quad 
    \bar x(\X) \theta_0^{L+1}\to \mathcal N(0, \infnngp(\X, \X)).  
\end{align}

Moreover $\bar x(\X)\theta_0^{L+1}$ and the term containing $f_0(\X)$ are the only stochastic term over the ensemble of network initializations, 
therefore for any $t$ the output $f(x^*)$ throughout training converges to a Gaussian distribution in the infinite width limit, with
\begin{align}
  \mathbb E [f_t(x^*)] &= \infnngp(x^*, \X)\infnngp^{-1}(I -e^{-\eta \infnngp t})\Y \,,
  \label{eq:NNGP-exact-dynamics-mean}
  \\
  {\rm Var}[f_t(x^*) ] 
    &= \infnngp(x^*, x^*) - \infnngp(x^*, \X)\infnngp^{-1}(I-e^{-2\eta \infnngp t})\infnngp(x^*, \X)^T \,.
    \label{eq:NNGP-exact-dynamics-var}
\end{align}

Thus the output of the neural network is also a GP and the asymptotic solution (i.e. $t\to\infty$) is identical to the posterior of the NNGP (\eqref{eq:nngp-exact-posterior}). Therefore, in the infinite width case, the optimized neural network is performing posterior sampling if only the readout layer is being trained. This result is a realization of sample-then-optimize equivalence identified in~\citet{matthews2017sample}.

\section{Computing NTK and NNGP Kernel}  
\label{sec:KernelDerivation}
For completeness, we reproduce, informally, the recursive formula of the NNGP kernel and the tangent kernel from \cite{lee2018deep} and \cite{Jacot2018ntk}, respectively. 
Let the activation function $\phi:\mathbb R\to\mathbb R$ be absolutely continuous. Let $\T$ and $\dot {\mathcal T}$ be functions from $2\times 2$ positive semi-definite matrices $\Sigma$ to $\mathbb R$ given by   
\begin{align} 
\begin{cases}
\T(\Sigma) = \mathbb E [\phi(u)\phi(v)]\,\,\,\, 
\\
\dot{\T}(\Sigma) = \mathbb E [\phi'(u)\phi'(v)]\,\,\,\, 
\end{cases}
(u, v)\sim \mathcal N(0, \Sigma)  \,.
\end{align}
In the infinite width limit, the NNGP and tangent kernel can be computed recursively.
Let $x, x'$ be two inputs in $\mathbb{R}^{n_0}$. 
Then $h^l(x)$ and $h^l(x')$ converge in distribution to a joint Gaussian as $\min\{n_1, \dots, n_{l-1}\}$. The mean is zero and the variance $\infnngp^{l}(x,x')$ is    
\begin{align}
\infnngp^{l}(x,x') = \tilde \infnngp^l(x, x')\otimes {\Id}_{n_l}
\end{align}
\begin{align}
\label{eq:sigma-map}
    \tilde\infnngp^{l}(x, x') =  \sigma_\omega^2 \T\left(\begin{bmatrix}
\tilde \infnngp^{l-1}(x, x) & \tilde \infnngp^{l-1}(x, x')
\\
\tilde \infnngp^{l-1}(x, x') & \tilde\infnngp^{l-1}(x', x')
\end{bmatrix}\right )     + \sigma_b^2 \,  
\end{align}
with base case 
\begin{align}
    \infnngp^1(x,x') = \sigma^2_{\omega} \cdot \frac{1}{n_0} x^T x' + \sigma^2_b.  
\end{align}
Using this one can also derive the tangent kernel for gradient descent training. 
We will use induction to show that 
\begin{align}\infntk^l (x,x') = \tilde \Theta^l(x,x') \otimes {\Id}_{n_l}
\end{align}
 where
\begin{align}\label{eq:tangent-kernel-recursive}
\tilde \Theta^l(x,x') = \tilde \infnngp^l(x,x')  +  \sigma_\omega^2 \tilde\Theta^{l-1}(x, x') \dot \T \left(\begin{bmatrix}
\tilde \infnngp^{l-1}(x, x) & \tilde \infnngp^{l-1}(x, x') 
\\
\tilde \infnngp^{l-1}(x, x') & \tilde\infnngp^{l-1}(x', x')
\end{bmatrix}\right ) 
\end{align}
with  $\tilde\infntk^1 = \tilde\infnngp^1$.  
Let 
\begin{align}
    J^l(x) = \nabla_{\theta^{\leq l}} h^l_0(x)  = [\nabla_{\theta^{l}} h^l_0(x), \nabla_{\theta^{< l}}  h^l_0(x) ] .  
\end{align}
Then  
\begin{align}
    J^l(x)J^l(x')^T &=  \nabla_{\theta^{l}} h^l_0(x) \nabla_{\theta^{l}} h^l_0(x')^T +  \nabla_{\theta^{< l}}  h^l_0(x) \nabla_{\theta^{< l}}  h^l_0(x')^T
\end{align}
Letting $n_1, \dots, n_{l-1}\to\infty$ sequentially, the first term converges to the NNGP kernel $\infnngp^l(x, x') $.
By applying the chain rule and the induction step (letting $n_1, \dots, n_{l-2} \to\infty$ sequentially), the second term is 
\begin{align}
    \nabla_{\theta^{< l}}  h^l_0(x) \nabla_{\theta^{< l}}  h^l_0(x')^T
    &=  
    \frac{\partial h^l_0(x)}{\partial h^{l-1}_0(x)}   
    \nabla_{\theta^{\leq l-1}}  h^{l-1}_0(x) \nabla_{\theta^{\leq  l-1}}  h^{l-1}_0(x')^T
                    \frac{\partial h^l_0(x')}{\partial h^{l-1}_0(x')}^T 
\\
&\to  \frac{\partial h^l_0(x)}{\partial h^{l-1}_0(x)}   
    \tilde\Theta^{l-1}(x, x')\otimes {\bf Id}_{n_{l-1}}
 \frac{\partial h^l_0(x')}{\partial h^{l-1}_0(x')}^T   \quad \quad    (n_1, \dots, n_{l-2} \to\infty)
 \\
 &\to
 \sigma_\omega^2   \left(\bbE \phi'(h_{0, i}^{l-1}(x)) \phi'(h_{0, i}^{l-1}(x')) \tilde\Theta^{l-1}(x, x')\right)\otimes {\Id}_{n_l} \quad \quad ( n_{l-1} \to\infty)
 \\
  &= 
  \left(\sigma_\omega^2  \tilde\Theta^{l-1}(x, x')  \dot \T \left(\begin{bmatrix}
\tilde \infnngp^{l-1}(x, x) & \tilde \infnngp^{l-1}(x, x')
\\
\tilde \infnngp^{l-1}(x, x') & \tilde\infnngp^{l-1}(x', x')
\end{bmatrix}\right )  \right)\otimes {\Id}_{n_l} 
\end{align}

\section{Results in function space for NTK parameterization transfer to standard parameterization}
\label{sec:compare-parameterization}
\begin{figure}[t]
\vskip 0.2in
 \centering
    \begin{subfigure}[b]{0.35\textwidth}
        \includegraphics[width=\textwidth]{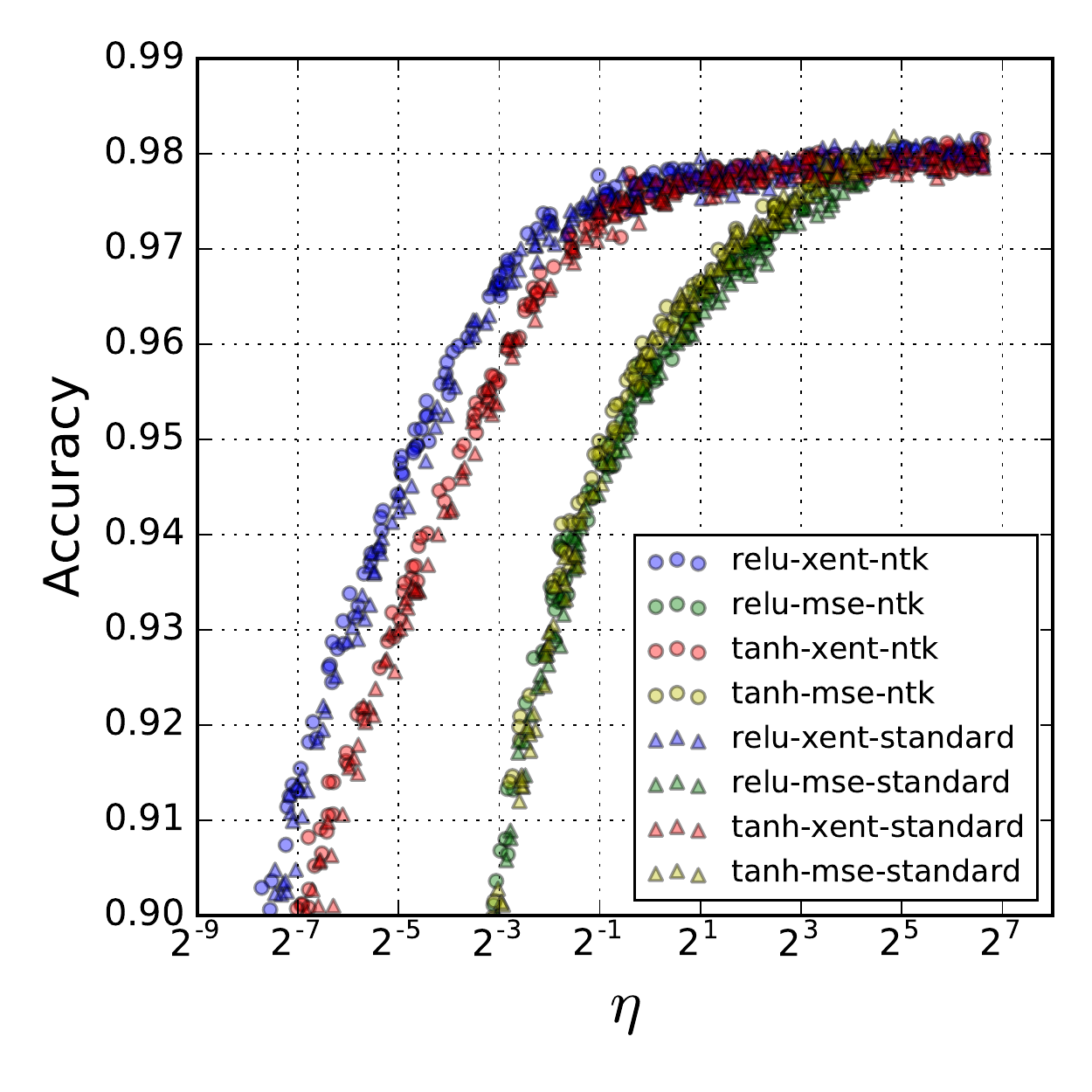}
        \caption{MNIST}
        \label{fig:mnist-ntk-vs-standard}
    \end{subfigure}
    \begin{subfigure}[b]{0.35\textwidth}
        \includegraphics[width=\textwidth]{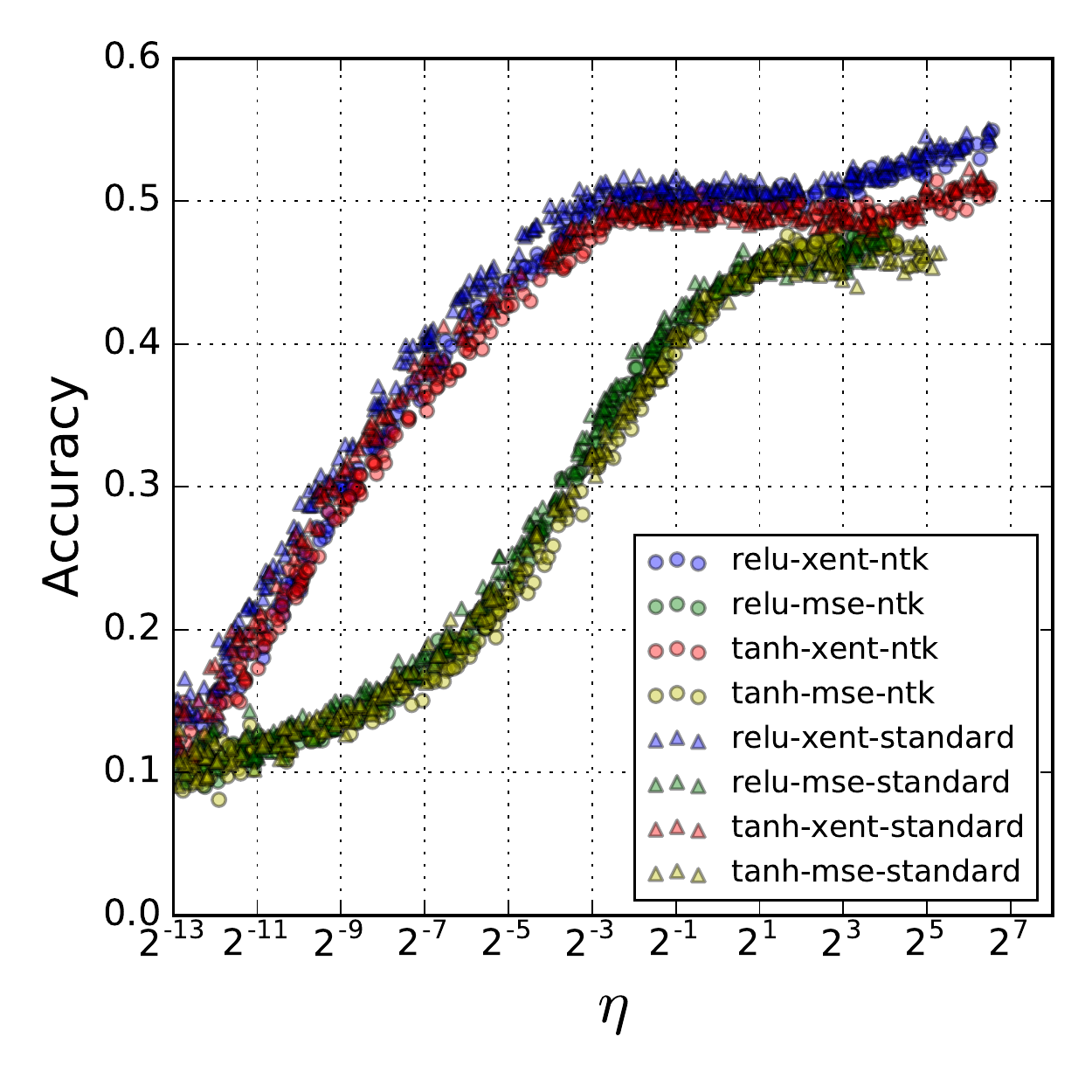}
        \caption{CIFAR}
        \label{fig:cifar-ntk-vs-standard}
    \end{subfigure}
\caption{{\bf NTK vs Standard parameterization.} Across different choices of dataset, activation function and loss function, models obtained from (S)GD training for both parameterization (circle and triangle denotes NTK and standard parameterization respectively) get similar performance.}
\label{fig:ntk-vs-standarad}
\vskip -0.2in
\end{figure}
In this Section we present a sketch for why the function space linearization results, derived in \cite{Jacot2018ntk} for NTK parameterized networks, also apply to networks with a standard parameterization. We follow this up with a formal proof in \sref{sec: converge proof} of the convergence of standard parameterization networks to their linearization in the limit of infinite width. A network with standard parameterization is described as:
\begin{align}
    \label{eq:recurrence-std}
    \begin{cases}
        h^{l+1}&=x^l W^{l+1} + b^{l+1}
        \\
        x^{l+1}&=\phi\left(h^{l+1}\right) 
        \end{cases}
        \,\, \textrm{and} 
        \,\,
        \begin{cases}
      W^{l}_{i, j}& =   \omega_{ij}^l  \sim \mathcal{N}\left(0, \frac {\sw^2}{n_l}\right)
        \\
        b_j^l &=  \beta_j^l  \sim \mathcal N\left(0, \sigma_b^2\right) 
    \end{cases} \,.
    \end{align}

The NTK parameterization in \eqref{eq:recurrence} is not commonly used for training neural networks. While the function that the network represents is the same for both NTK and standard parameterization, training dynamics under gradient descent are generally different for the two parameterizations. However, for a particular choice of layer-dependent learning rate
training dynamics also become identical. Let $\eta^l_{\text{NTK},w}$ and $\eta^l_{\text{NTK},b}$ be layer-dependent learning rate for $W^l$ and $b ^l$ in the NTK parameterization, and $\eta_\text{std} = \frac{1}{n_\text{max}} \eta_0$ be the learning rate for all parameters in the standard parameterization, where $n_\text{max} = \max_l n_l$. 
Recall that gradient descent training in standard neural networks requires a learning rate that scales with width like $\frac{1}{n_\text{max}}$, so $\eta_0$ defines a width-invariant learning rate \citep{parkoptimal}. 
If we choose
\begin{align}
\label{eq lr equiv}
\eta^l_\text{NTK, w} = \frac{n_l}{n_\text{max} \sw^2 }\eta_0, \qquad \text{and} \qquad
\eta^l_\text{NTK, b} = \frac{1}{n_\text{max}\sigma_b^2} \eta_0
,
\end{align}
then learning dynamics are identical for networks with NTK and standard parameterizations. 
With only extremely minor modifications, consisting of incorporating the multiplicative factors in Equation \ref{eq lr equiv} 
into the per-layer contributions to the Jacobian, the arguments in \sref{sec:Justification} go through for an NTK network with learning rates defined in Equation \ref{eq lr equiv}. 
Since an NTK network with these learning rates exhibits identical training dynamics to a standard network with learning rate $\eta_\text{std}$, the result in \sref{sec:Justification} that sufficiently wide NTK networks are linear in their parameters throughout training also applies to standard networks.

We can verify this property of networks with the standard parameterization experimentally. 
In Figure~\ref{fig:ntk-vs-standarad}, we see that for different choices of dataset, activation function and loss function, final performance of two different parameterization leads to similar quality model for similar value of normalized learning rate $\eta_{\textrm{std}} = \eta_{\textrm{NTK}} / n
$. Also, in Figure~\ref{fig:NTK-dynamics-standard}, we observe that our results is not due to the parameterization choice and holds for wide networks using the standard parameterization. 

\begin{figure}%
  \centering
  \includegraphics[width=\columnwidth]{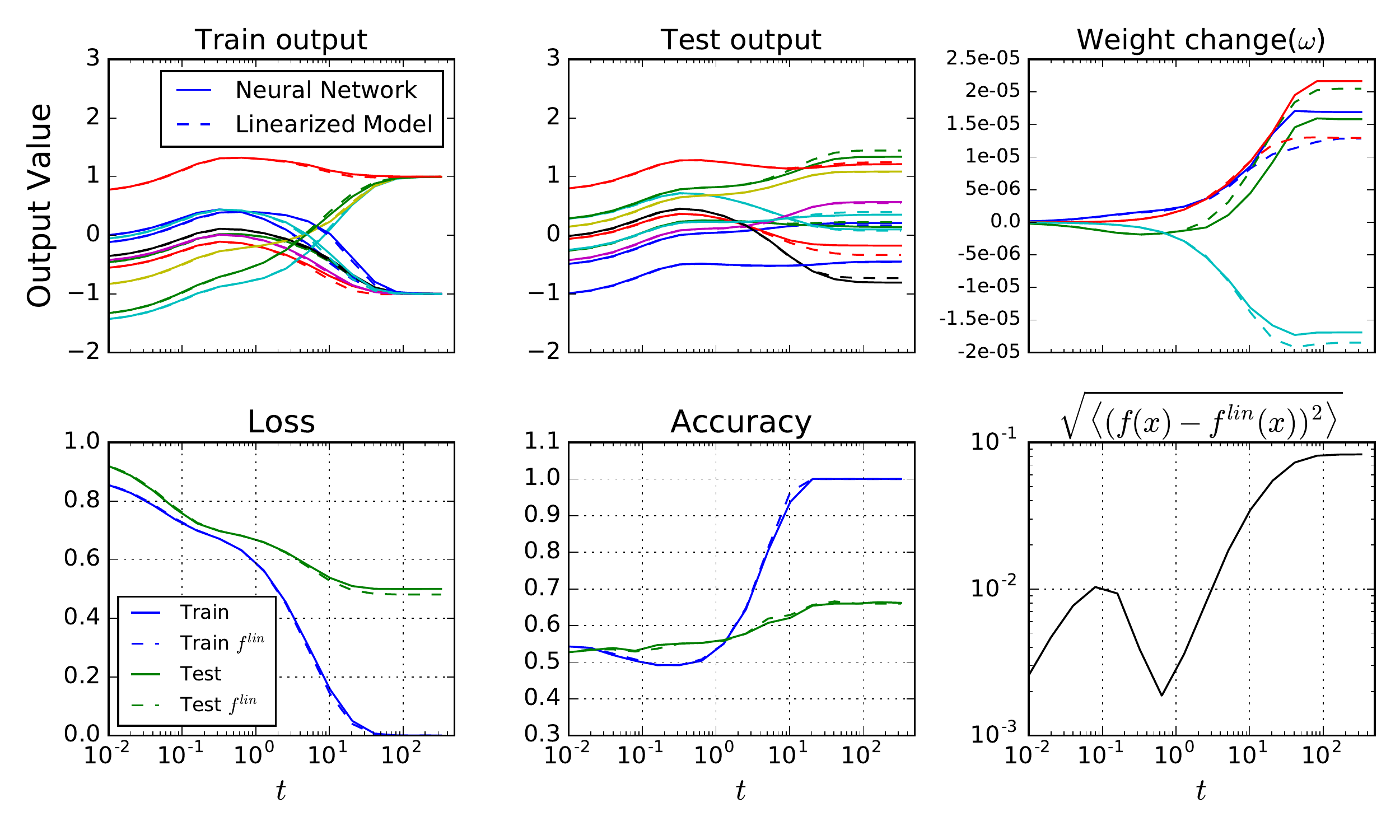}
  \caption{{\bf Exact and experimental dynamics are nearly identical for network outputs, and are similar for individual weights (Standard parameterization).} 
  Experiment is for an MSE loss, $\operatorname{ReLU}$ network with 5 hidden layers of width $n=2048$, $\eta = 0.005/2048$ $|\D|=256$, $k=1$, $\sigma_w^2=2.0$, and $\sigma_b^2=0.1$. All three panes in the first row show dynamics for a randomly selected subset of datapoints or parameters. First two panes in the second row show dynamics of loss and accuracy for training and test points agree well between original and linearized model. Bottom right pane shows the dynamics of RMSE between the two models on test points using empirical kernel.}
  \label{fig:NTK-dynamics-standard}
\end{figure}

\section{Convergence of neural network to its linearization, and stability of NTK under gradient descent}
\label{sec: converge proof}

In this section, we show that how to use the NTK to provide a simple proof of the global convergence of a neural network under (full-batch) gradient descent and the stability of NTK under gradient descent. 
We present the proof for standard parameterization. With some minor changes, the proof can also apply to the NTK parameterization. To lighten the notation, we only consider the asymptotic bound here. 
The neural networks are parameterized as in \eqref{eq:recurrence-std}.  
We make the following assumptions:
\\
\\ 
{\bf Assumptions [1-4]:}
\begin{enumerate}
    \item The widths of the hidden layers are identical, i.e. $n_1 = \dots = n_L = n$ (our proof extends naturally to the setting $\frac {n_l}{n_{l'}}\to \alpha_{l, l'}\in(0, \infty)$
    as $\min\{n_1, \dots , n_L\}\to \infty$.)    
    \item The analytic NTK $\Theta$ (defined in \eqref{eq:ntk-standard}) is full-rank, i.e. $0<\mins := \mins(\Theta) \leq \maxs :=\maxs(\Theta)<\infty .$ We set $\critical = 2 (\mins + \maxs)^{-1}$ . 
    \item The training set $(\X, \Y)$ is contained in some compact set and $x\neq \tilde x$ for all $x, \tilde x \in \X$.   
    \item The activation function $\phi$ satisfies 
    \begin{align}
        |\phi(0)|,\quad  \|\phi'\|_\infty, \quad \sup_{x\neq \tilde x} |\phi'(x) - \phi'(\tilde x)|/|x-\tilde x| < \infty. \label{eq:activation-assumption}
    \end{align}
\end{enumerate}
Assumption 2 indeed holds when $\X\subseteq \{x\in \mathbb R^{n_0}\}: \|x\|_2=1\}$ and 
$\phi(x)$ grows non-polynomially for large $x$ \cite{Jacot2018ntk}.  
Throughout this section, we use $C>0$ to denote some constant whose value may depend on 
$L$, $|\X|$ and $(\sigma_w^2, \sigma_b^2)$ and may change from line to line, but is always independent of $n$.  

Let $\theta_t$ denote the parameters at time step $t$.  We use the following short-hand
\begin{align}
    f(\theta_t) &= f(\X, \theta_t) \in\R^{|\X|\times k} 
    \\
    g(\theta_t) &= f(\X, \theta_t) - \Y \in\R^{|\X|\times k}
    \\
    J(\theta_t) &= \nabla_{\theta}f(\theta_t) \in\R^{(|\X| k)\times |\theta|}
\end{align}
where $|\X|$ is the cardinality of the training set and $k$ is the output dimension of the network.
The empirical and analytic NTK of the standard parameterization is defined as 
\begin{align}\label{eq:ntk-standard}
\begin{cases}
\finntk_t &:=\finntk_t(\X, \X) = \frac 1 n J(\theta_t)J(\theta_t)^T
\\
\infntk &:= \lim_{n\to\infty} \finntk_0  \quad {\rm in\quad  probability}. 
\end{cases}
\end{align}
Note that the convergence of the empirical NTK in probability is proved rigorously in \cite{yang2019scaling}.  
We consider the MSE loss
\begin{align}
    \mathcal L(t) = \frac 1 2 \|g(\theta_t)\|_2^2. 
\end{align}
 Since $f(\theta_t)$ converges in distribution to a mean zero Guassian with covariance $\infnngp$, one can show that for arbitrarily small $\delta_0>0$, there are constants $\lss>0$ and $n_0$ (both may depend on $\delta_0$, $|\X|$ and $\K$) such that for every $n \geq n_0$, with probability at least $(1 - \delta_0)$ over random initialization, 
\begin{align}
 \|g(\theta_0)\|_2 < \lss.  \label{eq:base-loss}
\end{align}

The gradient descent update with learning rate $\eta$ is 
\begin{align}
    \theta_{t+1} = \theta_t - \eta J(\theta_t)^Tg(\theta_t)
\end{align}
and the gradient flow equation is   
\begin{align}
    \dot\theta_{t} =  - J(\theta_t)^Tg(\theta_t)
    .
\end{align}
We prove convergence of neural network training and the stability of NTK for both discrete gradient descent and gradient flow.  
Both proofs rely on the local lipschitzness of the Jacobian $J(\theta)$.  
\begin{lemma}[{\bf Local Lipschitzness of the Jacobian}] \label{lemma:stability-jacobian} 
There is a $K>0$ such that for every $C>0$, with high probability over random initialization (w.h.p.o.r.i.) the following holds
\begin{align}\label{eq:jacobian-lip}
\begin{cases}  
    \frac 1 {\sqrt n}\|J(\theta) - J(\tilde \theta)\|_{F} &\leq K\|\theta - \tilde \theta\|_2
    \\
    \\
    \frac 1 {\sqrt n} \|J(\theta)\|_{F} & \leq K 
    \end{cases}
    , \quad \quad   \forall \theta, \, \tilde \theta \in B(\theta_0, C n^{-\frac 1 2})  
\end{align}
where 
\begin{align}
    B(\theta_0, R) := \{\theta: \|\theta-\theta_0\|_2 < R\}.   
\end{align}
\end{lemma}

The following are the main results of this section.   
\begin{theorem}[{\bf Gradient descent}]\label{thm:convergence}
Assume {\bf Assumptions [1-4]}.  
For $\delta_0>0$ and $\eta_0< \critical$, there exist $\lss>0$, $N\in\mathbb N$ and $K>1$, such that for every $n\geq N$, the following holds with probability at least $(1 - \delta_0)$ over random initialization when applying 
gradient descent with learning rate $\eta = \frac {\eta_0}{n}$,
\begin{align} \label{eq:exp-decay}
    \begin{cases}
    &\|g(\theta_{t})\|_2 \leq \left(1 - \frac {\eta_0 \mins}{3}\right)^t \lss  %
    \\
    \\
     &\sum_{j=1}^{t}\|\theta_j - \theta_{j-1}\|_2 \leq  \frac{\eta_0K\lss}{\sqrt n} \sum_{j=1}^{t} (1 - \frac {\eta_0 \mins}{3})^{j-1} 
 \leq  \frac {3K \lss}{\mins}  n^{-\frac 1 2} 
    \end{cases} 
\end{align}
and 
\begin{align}\label{eq:convergence-ntk} 
     \sup_{t} \|  \finntk_0 - \finntk_t\|_F  \leq \frac {6K^3\lss}{\mins}  n^{-\frac 1 2}\, . 
\end{align}
\end{theorem}

\begin{theorem}[{\bf Gradient Flow}]\label{thm:convergence-flow}
Assume {\bf Assumptions[1-4]}.   
For $\delta_0>0$, there exist $\lss>0$, $N\in\mathbb N$ and $K>1$, such that for every $n\geq N$, the following holds with probability at least $(1 - \delta_0)$ over random initialization when applying gradient flow with ``learning rate" $\eta = \frac {\eta_0}{n}$
\begin{align} \label{eq:exp-decay-flow}
    \begin{cases}
    &\|g(\theta_{t})\|_2 \leq e^{- \frac { \eta_0 \mins}{3}t} \lss
    \\
    \\
     &\|\theta_t - \theta_{0}\|_2 \leq  \frac {3K \lss}{\mins}(1 - e^{-\frac 1 3  \eta_0 \mins t})  n^{-\frac 1 2} 
    \end{cases} 
\end{align}
and 
\begin{align}\label{eq:convergence-ntk-flow} 
     \sup_{t} \|  \finntk_0 - \finntk_t\|_F  \leq \frac {6K^3\lss}{\mins}  n^{-\frac 1 2}\, . 
\end{align}
\end{theorem}
See the following two subsections for the proof. 
\begin{remark}
One can extend the results in Theorem \ref{thm:convergence} and Theorem \ref{thm:convergence-flow} to other architectures or functions as long as \begin{enumerate}
    \item The empirical NTK converges in probability and the limit is positive definite.  
    \item Lemma \ref{lemma:stability-jacobian} holds, i.e. the Jacobian is locally Lipschitz.    
\end{enumerate}   
\end{remark}
\subsection{Proof of Theorem \ref{thm:convergence}}\label{subsection:convergence-descent}
As discussed above, there exist $\lss$ and $n_0$ such that for every $n \geq n_0$, with probability at least $(1 - \delta_0/10)$ over random initialization, 
\begin{align}
 \|g(\theta_0)\|_2 < \lss  \label{eq:base-loss-1} \, . 
\end{align}
Let $C =  \frac {3K \lss}{\mins}$ in Lemma \ref{lemma:stability-jacobian}.  
We first prove \eqref{eq:exp-decay} by induction. 
Choose $n_1>n_0$ such that for every $n\geq n_1$ \eqref{eq:jacobian-lip} and \eqref{eq:base-loss-1} hold with probability at least $(1 -  \delta_0/5)$ over random initialization.  
The $t=0$ case is obvious and we assume \eqref{eq:exp-decay} holds for $t=t$. 
Then by induction and the second estimate of \eqref{eq:jacobian-lip}
\begin{align}
    \|\theta_{t+1} - \theta_t\|_2 \leq  \eta \|J(\theta_t)\|_{\op} \|g(\theta_t)\|_2 \leq \frac {K\eta_0}{\sqrt n} \left(1 - \frac {\eta_0 \mins}{3}\right)^t 
    \lss, 
\end{align}
which gives the first estimate of \eqref{eq:exp-decay} for $t+1$ and which also implies $\|\theta_{j} - \theta_0\|_2\leq  \frac {3K \lss}{\mins}n^{-\frac 1 2}$ for $j=0, \dots, t+1$. To prove the second one, we apply the mean value theorem and the formula for gradient decent update at step $t+1$ 
\begin{align}
    \|g(\theta_{t+1})\|_2 &= \|g(\theta_{t+1}) - g(\theta_{t}) + g(\theta_{t})\|_2
    \\
     &= \|J(\tilde \theta_t)(\theta_{t+1} - \theta_t) +  g(\theta_{t})\|_2  
    \\
    &= \|-\eta J(\tilde \theta_t)J(\theta_t)^T g(\theta_{t}) +  g(\theta_{t})\|_2 
    \\
    &\leq \|1 - \eta J(\tilde \theta_t)J(\theta_t)^T\|_{\op} \|g(\theta_t)\|_2
    \\
    &\leq  \|1 - \eta J(\tilde \theta_t)J(\theta_t)^T\|_{\op} \left(1 - \frac {\eta_0 \mins}{3}\right)^t \lss, 
\end{align}
where  $\tilde \theta_t$ is some linear interpolation between $ \theta_t$ and $ \theta_{t+1}$. It remains to show with probability at least $(1-\delta_0/2)$,   
\begin{align}
     \|1 - \eta J(\tilde \theta_t)J(\theta_t)^T\|_{\op} \leq  1 - \frac {\eta_0 \mins}{3}.  
\end{align}
This can be verified by Lemma \ref{lemma:stability-jacobian}. Because $\finntk_0\to \infntk$ \cite{yang2019scaling} in probability, one can find $n_2$ such that the event 
\begin{align}
    \|\infntk - \finntk_0\|_F \leq \frac {\eta_0\mins}{3}
\end{align}
has probability at least $(1-\delta_0/5)$ for every $n\geq n_2$. 
The assumption $\eta_0 < \frac 2{\mins+ \maxs}$ implies 
\begin{align}
    \|1 - \eta _0\infntk\|_{\op} \leq  1 - \eta_0\mins. 
\end{align}
Thus 
\begin{align}
    &\|1 - \eta J(\tilde \theta_t)J(\theta_t)^T\|_{\op}
    \\ \leq&
    \|1 - \eta _0\infntk\|_{\op} + \eta_0\|\infntk - \finntk_0\|_{\op} + \eta\|J(\theta_0)J(\theta_0)^T-J(\tilde \theta_t)J(\theta_t)^T\|_{\op}
    \\
    \leq & 1 - \eta_0\mins + \frac {\eta_0\mins}{3} + \eta_0 K^2 (\|\theta_t - \theta_0\|_2 + \|\tilde \theta_t - \theta_0\|_2)
    \\
    \leq &  1 - \eta_0\mins + \frac {\eta_0\mins}{3} + 2 \eta_0 K^2 \frac {3K \lss}{\mins} \frac 1  {\sqrt n}   \leq  1 - \frac {\eta_0 \mins}{3}
\end{align}
with probability as least $(1 -\delta_0/2)$ if 
\begin{align}
    n \geq \left(\frac {18K^3 \lss}{\mins^2 }\right)^2.    
\end{align}
Therefore, we only need to set  
\begin{align}
    N = \max\left\{n_0, n_1, n_2, \left(\frac {18K^3 \lss}{\mins^2 }\right)^2 \right\}.  
\end{align}
To verify \eqref{eq:convergence-ntk}, notice that  
\begin{align}
      \|  \finntk_0 - \finntk_t\|_F &=   \frac 1 n \| J(\theta_0) J(\theta_0) ^T - J(\theta_t) J(\theta_t)^T \|_F 
     \\
     &\leq  \frac 1 n  \left(\| J(\theta_0)\|_{\op} \|J(\theta_0) ^T - J(\theta_t)^T \|_F 
      + \| J(\theta_t) - J(\theta_0)\|_{\op} \|J(\theta_t)^T\|_F\right) 
      \\
      &\leq 2K^2 \|\theta_0 - \theta_t\|_2  
      \\
      & \leq  \frac {6K^3\lss}{\mins} \frac 1  {\sqrt n}, 
\end{align}
where we have applied the second estimate of \eqref{eq:exp-decay} and \eqref{eq:jacobian-lip}.

\subsection{Proof of Theorem \ref{thm:convergence-flow}} \label{subsection:proof-gradient-flow}
The first step is the same. There exist $\lss$ and $n_0$ such that for every $n \geq n_0$, with probability at least $(1 - \delta_0/10)$ over random initialization, 
\begin{align}
 \|g(\theta_0)\|_2 < \lss  \label{eq:base-loss-1-flow} \, . 
\end{align}
Let $C =  \frac {3K \lss}{\mins}$ in Lemma \ref{lemma:stability-jacobian}. Using the same arguments as in Section \ref{subsection:convergence-descent}, one can show that there exists $n_1$ such that for all $n\geq n_1$, with probability at least 
$(1 - \delta_0/10) $ 
\begin{align}
    \frac 1 n J(\theta)J(\theta)^T \succ \frac 1 3  \mins {\Id}  \quad \forall \theta \in B(\theta_0, Cn^{-\frac 1 2})
\end{align}
Let 
\begin{align}
    t_1 = \inf \left\{t: \|\theta_t - \theta_{0}\|_2 \geq  \frac {3K \lss}{\mins}  n^{-\frac 1 2}  \right\}
\end{align}
We claim $t_1=\infty$. If not, then for all $t\leq t_1$, $\theta_t\in B(\theta_0, Cn^{-\frac 1 2})$ and   
\begin{align}
    \finntk_t \succ \frac 1 3  \mins \Id . 
\end{align}
Thus 
\begin{align}
    \frac d {dt} \left ( \|g(t)\|_2^2\right) = - 2\eta_0g(t)^T \finntk_t g(t) \leq - \frac 2 3 \eta_0 \mins \|g(t)\|_2^2 
\end{align}
and 
\begin{align} \label{eq:useful-2}
    \|g(t)\|_2^2 \leq  e^{-\frac 2 3 \eta_0\mins t} \|g(0)\|_2^2 \leq  e^{-\frac 2 3 \eta_0\mins t} \lss^2 . 
\end{align}
Note that 
\begin{align}
    \frac d {dt} \|\theta_t - \theta_0\|_2 \leq \left\| \frac d {dt} \theta_t\right\|_2 = \frac {\eta_0} {n}\|J(\theta_t)g(t)\|_2 \leq  {\eta_0} K\lss e^{-\frac 1 3 \eta_0 \mins t}  n^{-1/2}
\end{align}
which implies, for all $t\leq t_1$ 
\begin{align}
    \|\theta_t - \theta_0\|_2 \leq   \frac {3K \lss}{\mins}(1 - e^{-\frac 1 3  \eta_0 \mins t})  n^{-\frac 1 2}
    \leq   \frac {3K \lss}{\mins}(1 - e^{-\frac 1 3  \eta_0 \mins t_1})  n^{-\frac 1 2} < \frac {3K \lss}{\mins} n^{-\frac 1 2} \,.
\end{align}
This contradicts to the definition of $t_1$ and thus $t_1=\infty$. Note that \eqref{eq:useful-2} is the same as the first equation of \eqref{eq:exp-decay-flow}.

\subsection{Proof of Lemma \ref{lemma:stability-jacobian}}
The proof relies on upper bounds of operator norms of random Gaussian matrices.  
\begin{theorem}[Corollary 5.35 \cite{vershynin2010introduction}]\label{thm:operator-bound-random-gaussian}
Let $A = A_{N, n}$ be an $N \times n$ random matrix whose
entries are independent standard normal random variables. Then for every $t \geq 0$,
with probability at least $1-2 \exp(-t^2/2)$  one has
\begin{align}
\sqrt N - \sqrt n - t \leq     \mins(A) \leq   \maxs (A) \leq \sqrt N + \sqrt n + t.  
\end{align}
\end{theorem}
For $l\geq 1$, let 
\begin{align}
   & \delta^l(\theta, x) := \nabla_{h^l(\theta, x)} f^{L+1}(\theta, x)\in \R^{kn}
    \\
    &\delta^l(\theta, \X) := \nabla_{h^l(\theta, \X)} f^{L+1}(\theta, \X)\in \R^{(k\times |\X|)\times (n\times \X)}
\end{align}

Let $\theta = \{W^l, b^l\}$ and $\tilde \theta = \{\tilde W^l, \tilde b^l\}$ be any two points in $B(\theta_0, \frac C {\sqrt n})$.
By the above theorem and the triangle inequality, w.h.p. over random initialization,    
\begin{align}
    \| W^1\|_{\op}, \quad \|\tilde W^1\|_{\op} \leq  3\sw \frac{\sqrt n} { \sqrt {n_0}}, \quad  \| W^l\|_{\op}, \quad \|\tilde W^l\|_{\op} \leq 3\sw \quad {\rm for }\quad 2\leq l \leq L+1 %
\end{align}
Using this and the assumption on $\phi$ \eqref{eq:activation-assumption},  it is not difficult to show that there is a constant $K_1$, depending on $\sw^2, \sigma_b^2, |\X|$ and $L$ such that with high probability over random initialization\footnote{These two estimates can be obtained via induction. To prove bounds relating to $x^l$ and $\delta^l$, one starts with $l=1$ and $l=L$, respectively.}  
\begin{align}
     n^{-\frac 1 2} \|x^l(\theta, \X)\|_{2}, \quad  \|\delta^l(\theta, \X )\|_{2} &\leq K_1,
    \\ 
     n^{-\frac 1 2} \|x^l(\theta, \X)- x^l(\tilde \theta, \X) \|_{2}, \quad   
    \|\delta^l(\theta, \X ) - \delta^l(\tilde \theta, \X )\|_{2} &\leq K_1\|\tilde \theta - \theta\|_2
\end{align}
Lemma \ref{lemma:stability-jacobian} follows from these two estimates. Indeed, with high probability over random initialization   
\begin{align}
    \|J(\theta)\|_{F}^2 &= \sum_l \|J(W^l)\|_F^2 + \|J(b^l)\|_F^2
    \\
    &= \sum_l \sum_{x\in \X}\|x^{l-1}(\theta, x)\delta^l(\theta, x)^T\|_F^2 + \|\delta^l(\theta, x)^T\|_F^2
    \\
    &\leq \sum_l \sum_{x\in \X}
    (1+ \|x^{l-1} (\theta, x)\|_F^2)\|\delta^l(\theta, x)^T\|_F^2 
    \\
    &\leq \sum_l (1+ K_1^2n)\sum_x \|\delta^l(\theta, x)^T\|_F^2 
    \\ &\leq \sum_l K_1^2(1+ K_1^2 n)
    \\
    & \leq 2(L+1)K_1^4 n,  
\end{align}
and similarly  
\begin{align}
    &\|J(\theta)- J(\tilde\theta)\|_{F}^2
    \\
    =& \sum_l \sum_{x\in \X}\|x^{l-1}(\theta, x)\delta^l(\theta, x)^T
    - x^{l-1}(\tilde\theta, x)\delta^l(\tilde\theta, x)^T\|_F^2 + \|\delta^l(\theta, x)^T- \delta^l(\tilde\theta, x)^T\|_F^2
    \\
    \leq &\left( \sum_l \left(K_1^4n  + K_1^4n \right)  + K_1^2  \right) \| \theta - \tilde{\theta} \|_2 
    \\ \leq & 3(L+1)K_1^4 n \, \| \theta - \tilde{\theta} \|_2. 
\end{align}

\subsection{Remarks on NTK parameterization}
For completeness, we also include analogues of Theorem \ref{thm:convergence} and Lemma \ref{lemma:stability-jacobian} with NTK parameterization.  
\begin{theorem}[NTK parameterization]
Assume {\bf Assumptions [1-4]}.  
For $\delta_0>0$ and $\eta_0< \critical$, there exist $\lss>0$, $N\in\mathbb N$ and $K>1$, such that for every $n\geq N$, the following holds with probability at least $(1 - \delta_0)$ over random initialization when applying 
gradient descent with learning rate $\eta =  {\eta_0}$,
\begin{align}
    \begin{cases}
    &\|g(\theta_{t})\|_2 \leq \left(1 - \frac {\eta_0 \mins}{3}\right)^t \lss  %
    \\
    \\
     &\sum_{j=1}^{t}\|\theta_j - \theta_{j-1}\|_2 \leq  {K\eta_0} \sum_{j=1}^{t} (1 - \frac {\eta_0 \mins}{3})^{j-1} 
    \lss \leq  \frac {3K \lss}{\mins} 
    \end{cases} 
\end{align}
and 
\begin{align} 
     \sup_{t} \|  \finntk_0 - \finntk_t\|_F  \leq \frac {6K^3\lss}{\mins}  n^{-\frac 1 2}\, . 
\end{align}
\end{theorem}
\begin{lemma}[NTK parameterization: Local Lipschitzness of the Jacobian] 
There is a $K>0$ such that for every $C>0$, with high probability over random initialization the following holds
\begin{align}
\begin{cases}  
    \|J(\theta) - J(\tilde \theta)\|_{F} &\leq K\|\theta - \tilde \theta\|_2
    \\
    \\
     \|J(\theta)\|_{F} & \leq K 
    \end{cases}
    , \quad \quad   \forall \theta, \, \tilde \theta \in B(\theta_0, C)  
\end{align}
\end{lemma}

\section{Bounding the discrepancy between the original and the linearized network: MSE loss}
\label{sec:sup-discrepancy}
We provide the proof for the gradient flow case. The proof for gradient descent can be obtained similarly.
To simplify the notation, let $\glin(t) \equiv \flin_t(\X) - \Y$ and $ g(t) \equiv f_t(\X) - \Y$. The theorem and proof apply to both standard and NTK parameterization. We use the notation $\lesssim$ to hide the dependence on uninteresting constants. 
\begin{theorem}
Same as in Theorem \ref{thm:convergence-flow}. For every $x\in\mathbb R^{n_0}$ with $\|x\|_2\leq 1$, for $\delta_0>0$ arbitrarily small, there exist $ R_0>0$ and $N\in\mathbb N$ such that for every $n\geq N$,  with probability at least $(1-\delta_0)$ over random initialization,   
\begin{align}
 \sup_{t}\left\|\glin(t)  - g(t)\right\|_2\, ,\quad \sup_{t}\left\|\glin(t, x)  - g(t, x)\right\|_2   \lesssim  n^{-\frac 1 2} R_0^2.   
\end{align}
\end{theorem}
\begin{proof}
\begin{align}
&\frac {d }{dt}\left( \exp(\eta_0 \finntk_0 t)(\glin(t) - g(t) )\right)
\\
 = &\eta_0 \left(\finntk_0 \exp(\eta_0 \finntk_0 t)(\glin(t)  -g(t) ) 
+ \exp(\eta_0\finntk_0 t)(-\finntk_0 \glin(t)  + \finntk_t g(t) )\right) 
\\
= & \eta_0\left(\exp( \eta_0 \finntk_0 t)(\finntk_t - \finntk_0)g(t) \right)
\end{align}
Integrating both sides and using the fact $\glin(0) = g(0)$, 
\begin{align}
 (\glin(t)  - g(t) )
  = 
  -&\int_{0}^t  \eta_0\left(\exp(\eta_0\finntk_0 (s-t))(\finntk_s - \finntk_0)(\glin(s) - g(s))\right) ds 
  \\
  +&\int_{0}^t  \eta_0\left(\exp(\eta_0\finntk_0 (s-t))(\finntk_s - \finntk_0)\glin(s) \right) ds 
\end{align}
Let $\lambda_0>0$ be the smallest eigenvalue of $\finntk_0$ (with high probability $\lambda_0 >\frac 1  3\mins $). Taking the norm gives 
\begin{align}
 \|\glin(t) - g(t)\|_2
  \leq  
  &
   \eta_0 \Big(\int_{0}^t  \|\exp(\finntk_0 \eta_0(s-t))\|_{op}\|(\finntk_s - \finntk_0)\|_{op}\|\glin(s) - g(s)\|_2 ds 
  \\
  &+\int_{0}^t\|\exp(\finntk_0 \eta_0(s-t))\|_{op}\|(\finntk_s - \finntk_0)\|_{op} \|\glin(s)\|_2  ds 
  \Big) 
  \\
  \leq  
  &\eta_0 \Big(\int_{0}^t  e^{\eta_0\lambda_0(s-t)}\|(\finntk_s - \finntk_0)\|_{op}\|\glin(s) - g(s)\|_2 ds 
  \\   
  &+\int_{0}^t e^{\eta_0\lambda_0(s-t)} \|(\finntk_s - \finntk_0)\|_{op} \|\glin(s)\|_2  ds \Big) 
\end{align}

Let  
\begin{align}
    u(t) &\equiv e^{\lambda_0 \eta_0 t} \|\glin(t) - g(t)\|_2
    \\
    \alpha(t) &\equiv \eta_0 \int_{0}^t e^{\lambda_0 \eta_0 s } \|(\finntk_s - \finntk_0)\|_{op} \|\glin(s)\|_2  ds 
    \\ 
    \beta(t) &\equiv  \eta_0\|(\finntk_t - \finntk_0)\|_{op} 
\end{align}
The above can be written as 
\begin{align}
    u(t) \leq  \alpha(t) + \int_{0}^{t}\beta(s) u(s)ds  
\end{align}
Note that $\alpha(t)$ is non-decreasing. Applying an integral form of the Gr\"{o}nwall's inequality (see Theorem 1 in \cite{dragomir2003some})  gives 
\begin{align}
u(t) \leq \alpha(t) \exp \left({\int_{0}^t \beta(s)ds} \right)
\end{align}
Note that 
\begin{align}
\|\glin(t)\|_2 = \|\exp\left( -\eta_0 \finntk_0 t\right) \glin(0)\|_2 \leq \|\exp\left( -\eta_0 \finntk_0 t\right)\|_{op} \|\glin(0)\|_2 =  e^{-\lambda_0 \eta_0 t} \|\glin(0)\|_2 \,.
\end{align}
Then 
\begin{align}
    \|\glin(t)-g(t)\|_2 &\leq 
    \eta_0 e^{-\lambda_0 \eta_0 t}\int_{0}^t e^{\lambda_0 \eta_0 s }
    \|\finntk_s - \finntk_0\|_{op} \|\glin(s)\|_2  ds 
    \exp\left({\int_{0}^t \eta_0 \|\finntk_s - \finntk_0\|_{op}ds}\right)
    \\
    &\leq
    \eta_0 e^{-\lambda_0 \eta_0 t}\|\glin(0)\|_2 \int_{0}^t
    \|(\finntk_s - \finntk_0)\|_{op}ds 
    \exp \left({\int_{0}^t \eta_0 \|\finntk_s - \finntk_0\|_{op}ds}\right)
\end{align}
Let $\sigma_t = \sup_{0\leq s\leq t}\|\finntk_s - \finntk_0\|_{op}$. Then  
\begin{align}\label{eq:useful}
     \|\glin(t)-g(t)\|_2 \lesssim  \left( \eta_0 t {\sigma_t}  e^{-\lambda_0 \eta_0 t + \sigma_t\eta_0 t}\right) \|\glin(0)\|_2 
\end{align}
As it is proved in Theorem \ref{thm:convergence}, for every $\delta_0 >0$, with probability at least $(1-\delta_0)$ over random initialization,   
\begin{align} 
\sup_{t} \sigma_t \leq \sup_{t} \|  \finntk_0 - \finntk_t\|_F \lesssim  n^{-1/2}R_0 \to  0 \, \label{eq: sigma-decay}
\end{align}
when $n_1=\dots =n_L=n\to\infty$. 
Thus for large $n$ and any polynomial $P(t)$ (we use $P(t) = t$ here) 
\begin{align}
 \sup_{t} e^{-\lambda_0 \eta_0 t + \sigma_t \eta_0 t } \eta_0 P(t)   =\mathcal  O(1)    
\end{align}
Therefore 
\begin{align}
\label{eq:discrepancy-training-appendix} 
    \sup_{t }\|\glin(t)-g(t)\|_2 \lesssim \sup_t \sigma_t R_0 \lesssim n^{-1/2} R_0^2\to 0 \,, 
\end{align}
as  $n \to \infty$.

Now we control the discrepancy on a test point $x$. Let $y$ be its true label. Similarly, 
\begin{align}
\frac d {dt} \left(\glin(t, x)  - g(t, x)\right) =  - \eta_0  \left(\finntk_0(x, \X)- \finntk_t(x, \X)\right) \glin(t)
 + \eta_0 \finntk_t(x, \X)(g(t) - \glin(t)).   
\end{align}
Integrating over $[0,t]$ and taking the norm imply  
\begin{align}
    &\left\|\glin(t, x)  - g(t, x)\right\|_2 
\\    \leq&  \eta_0 \int_0^t 
    \left\|\finntk_0(x, \X)- \finntk_s(x, \X)\right\|_2 
    \| \glin(s)\|_2 ds 
 + \eta_0 \int_0^t\|\finntk_s(x, \X) \|_2 \|g(s) - \glin(s)\|_2ds 
 \\
 \leq& \eta_0 
 \|\glin(0)\|_2  \int_0^t  \left\|\finntk_0(x, \X)- \finntk_s(x, \X)\right\|_2 
 e^{-\eta_0 \lambda_0 s} ds \label{eq: first-bound}
 \\ 
  & + \eta_0 \int_0^t(\|\finntk_0(x, \X)\|_2 + \|\finntk_s(x, \X) - \finntk_0(x,\X)\|_2) 
  \|g(s) - \glin(s)\|_2ds \label{eq: second bound}
\end{align}
Similarly, Lemma \ref{lemma:stability-jacobian} implies 
\begin{align}
    \sup_{t}\left\|\finntk_0(x, \X)- \finntk_t(x, \X)\right\|_2 \lesssim n^{-\frac 1 2} R_0
\end{align}
This gives 
\begin{align}
    \textrm{ }(\ref{eq: first-bound}) \lesssim n^{-\frac 1 2} R_0^2. 
\end{align}
Using \eqref{eq:useful} and \eqref{eq: sigma-decay},  
\begin{align}
    \textrm{(\ref{eq: second bound})} \lesssim
    \|\finntk_0(x, \X)\|_2 \int_0^t \left( \eta_0 s {\sigma_s}  e^{-\lambda_0 \eta_0 s + \sigma_s\eta_0 s}\right) \|\glin(0)\|_2 dt   \lesssim n^{-\frac 1 2}\,.
\end{align}

\end{proof}

\section{Convergence of empirical kernel}
\label{sec kernel converge}

As in \citet{novak2018bayesian}, we can use Monte Carlo estimates of the tangent kernel (\eqref{eq:tangent-kernel}) to probe convergence to the infinite width kernel (analytically computed using Equations \ref{eq:sigma-map}, \ref{eq:tangent-kernel-recursive}). 
For simplicity, we consider random inputs drawn from ${\mathcal N}(0, 1)$ with $n_0=1024$. In Figure~\ref{fig:convergence-vs-width-d3}, we observe convergence as both width $n$ increases and the number of Monte Carlo samples $M$ increases. 
For both NNGP and tangent kernels we observe $\|\finntk^{(n)} - \infntk\|_F = \mathcal O\left(1/\sqrt{n}\right)$
and $\|\finnngp^{(n)} - \infnngp\|_F = \mathcal O\left({1}/\sqrt{n}\right)$, as predicted by a CLT in \citet{daniely2016}.

\begin{figure}%
  \centering
  \includegraphics[width=\columnwidth]{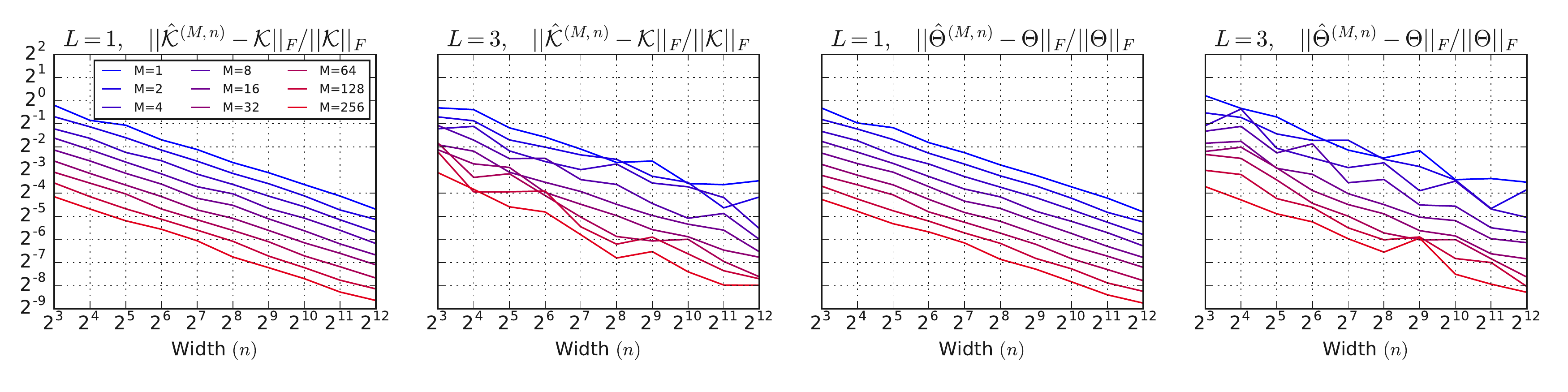}
  \caption{\textbf{Kernel convergence.} Kernels computed from randomly initialized \relu{} networks with one and three hidden layers converge to the corresponding analytic kernel as width $n$ and number of Monte Carlo samples $M$ increases. Colors indicate averages over different numbers of Monte Carlo samples.}
  \label{fig:convergence-vs-width-d3}
\end{figure}
\begin{figure}[ht]
  \centering
  \includegraphics[width=\columnwidth]{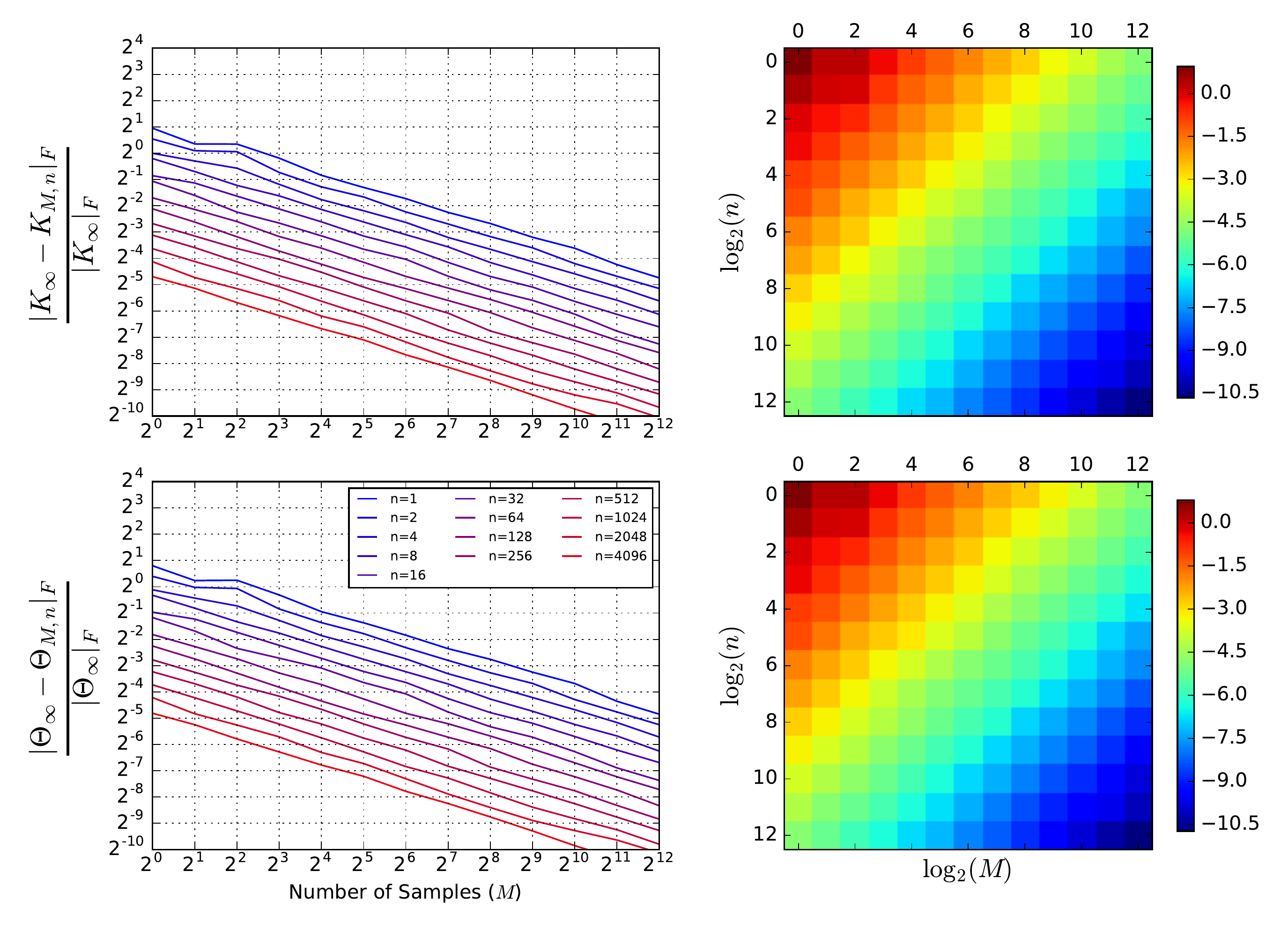}
  \caption{\textbf{Kernel convergence.} Kernels from single hidden layer randomly initialized \relu{} network convergence to analytic kernel using Monte Carlo sampling ($M$ samples). See \sref{sec kernel converge} for additional discussion.}
  \label{fig:convergence-vs-width2}
\end{figure}

\section{Details on Wide Residual Network}
\begin{table}[ht]
  \caption{{\bf Wide Residual Network architecture from~\citet{zagoruyko2016wide}}. In the residual block, we follow Batch Normalization-ReLU-Conv ordering.}
  \label{tab:wide_resnet_config}
  \centering
    \begin{tabular}{ccc}
    \toprule
    group name & output size & block type \\
    \midrule
    conv1    &  32 $\times$ 32  & [3$\times$3, \textrm{channel size}] \\
    conv2    &  32 $\times$ 32  & $\begin{bmatrix} 3\times3,& \textrm{channel size}\\ 3\times3,& \textrm{channel size} \end{bmatrix}$ $\times$ N\\
    conv3    &  16 $\times$ 16  & $\begin{bmatrix} 3\times3,& \textrm{channel size}\\ 3\times3,& \textrm{channel size} \end{bmatrix}$ $\times$ N\\
    conv4    &  8 $\times$ 8  & $\begin{bmatrix} 3\times3,& \textrm{channel size}\\ 3\times3,& \textrm{channel size} \end{bmatrix}$ $\times$ N\\
    avg-pool & 1 $\times$ 1 & [8 $\times$ 8]\\
    \bottomrule
    \end{tabular}
\end{table}

\end{document}